\documentclass[11pt]{article}
\usepackage{preamble}

\usepackage[margin=1in]{geometry}

\newcommand{\Sig}{\vec{\Sigma}}
\newcommand{\calY}{\mathcal{Y}}
\newcommand{\B}{B}
\newcommand{\radius}{R}

\title{Efficiently Learning Any One Hidden Layer ReLU Network From Queries}
\author{Sitan Chen\thanks{This work was supported in part by an NSF CAREER Award CCF-1453261 and NSF Large CCF-1565235.} \\
\texttt{sitanc@mit.edu}\\
MIT
\and Adam R. Klivans\thanks{Supported by NSF awards AF-1909204, AF-1717896, and the NSF AI Institute for Foundations of Machine Learning (IFML).} \\
\texttt{klivans@cs.utexas.edu} \\
UT Austin
\and Raghu Meka\thanks{Supported by NSF CAREER Award CCF-1553605.} \\
\texttt{raghum@cs.ucla.edu} \\
UCLA
}

\begin{document}

\maketitle

\begin{abstract}
    Model extraction attacks have renewed interest in the classic problem of learning neural networks from queries.  This work gives the first polynomial-time algorithm for learning one hidden layer neural networks provided black-box access to the network. Formally, we show that if $F$ is an arbitrary one hidden layer neural network with ReLU activations, there is an algorithm with query complexity and running time that is polynomial in all parameters that outputs a network $F'$ achieving low square loss relative to $F$ with respect to the Gaussian measure. While a number of works in the security literature have proposed and empirically demonstrated the effectiveness of certain algorithms for this problem, ours is the first with fully polynomial-time guarantees of efficiency for worst-case networks (in particular our algorithm succeeds in the overparameterized setting).
\end{abstract}

\section{Introduction}
The problem of learning neural networks given random labeled examples continues to be a fundamental algorithmic challenge in machine learning theory. There has been a resurgence of interest in developing algorithms for learning one hidden layer neural networks, namely functions of the form $F(x) = \sum^{k}_{i = 1} s_i \sigma(\iprod{w_i,x} - b_i)$ where $w_i \in \R^d$, $s_i\in\brc{\pm 1}$, $b_i \in \R$ are the unknown parameters.  Here we consider the ubiquitous setting where $\sigma$ is the ReLU activation, the input distribution on $x$ is Gaussian, and the goal is to output a ReLU network $F'$ with $\mathbb{E}[(F(x) - F'(x))^2] \leq \epsilon.$

Despite much recent work on this problem \cite{janzamin2015beating,zhong2017recovery,convotron,ge2018learning,ge2018learning2,zgu,bakshi2019learning,diakonikolas2020algorithms,CKM,LiMaZhang}, obtaining a polynomial-time algorithm remains open. All known works take additional assumptions on the unknown weights and coefficients or do not run in polynomial-time in all the important parameters. In fact, there are a number of lower bounds, both for restricted models of computation like correlational statistical queries \cite{goel2020superpolynomial,diakonikolas2020algorithms} as well as under various average-case assumptions \cite{daniely2020hardness,danielyprg}, that suggest that a truly polynomial-time algorithm may be impossible to achieve.

In this work, we strengthen the learner by allowing it query access to the unknown network.  That is, the learner may select an input $x$ of its choosing and receive the output value of the unknown network $F(x)$.  Our main result is the first polynomial-time algorithm for learning one hidden layer ReLU networks with respect to Gaussian inputs when the learner is given query access:

\begin{theorem} \label{thm:main}
Let $F(x) = \sum^{k}_{i = 1} s_i \sigma(\iprod{w_i,x} - b_i)$ with $\|w_i\|_2 \leq R$, $s_i\in\brc{\pm 1}$, and $b_i \leq B$.  Given black-box access to $F$, there exists a (randomized) algorithm that will output with probability at least $1-\delta$, a one hidden layer network $F'$ such that $\mathbb{E}_{x\sim\calN(0,\mathrm{Id})}[(F(x) - F'(x))^2] \leq \epsilon$.  The algorithm has query complexity and running time that is polynomial in $d,B,R,k, 1/\epsilon$, and $\log(1/\delta)$.
\end{theorem}

In light of the aforementioned lower bounds \cite{goel2020superpolynomial,diakonikolas2020algorithms,daniely2020hardness,danielyprg}, Theorem~\ref{thm:main} suggests there is a separation between learning from random samples and learning from queries in the context of neural networks.

In addition to being a naturally motivated challenge in computational learning theory, the question of learning neural networks from queries has also received strong interest from the security and privacy communities in light of so-called \emph{model extraction attacks} \cite{TramerZJRR16,MilliSDH19,Papernot,JagielskiCBKP20,RolnickK20,jayaram2020span}.  These attacks attempt to reverse-engineer neural networks found in publicly deployed real-world systems. Since the target network is publicly available, an attacker therefore obtains black-box query access to the network \--- exactly the learning model we are working in. Of particular interest is the work of Carlini et al.~\cite{CarliniJM20} (see also \cite{JagielskiCBKP20}), who gave several heuristics for learning deep neural networks given black-box access (they also empirically verified their results on networks trained on MNIST). Another relevant work is that of \cite{MilliSDH19} which gave theoretical guarantees for the problem we study under quite strong separation/linear independence assumptions on the weight vectors; in fact they are able to \emph{exactly} recover the network, but in the absence of such strong assumptions this is impossible (see Section~\ref{sec:compare}). A very recent follow-up work \cite{daniely2021exact} improved upon the results of \cite{MilliSDH19} by giving a similar guarantee under a somewhat milder assumption and notably also giving an algorithm for learning ReLU networks with \emph{two} hidden layers under such assumptions. Since these papers bear strong parallels with our techniques, in Section~\ref{sec:compare} we provide a more detailed description of why these approaches break down in our setting and highlight the subtleties that we address to achieve the guarantee of Theorem~\ref{thm:main}.

\subsection{Our Approach}

Here we describe our approach at a high level. Fix an unknown one hidden-layer network $F(x) = \sum^{k}_{i = 1} s_i\sigma(\iprod{w_i,x} - b_i)$ with weight vectors $w_i \in \R^d$, signs $s_i\in\brc{\pm 1}$, and bias terms $b_i \in \R$. Consider a line $L = \{x_0 + t \cdot v\}_{t \in \R}$ for $x_0,v\in\R^d$ and the \emph{restriction} $F|_L(t)\triangleq F(x_0 + t\cdot v)$ given by \begin{equation}
		F|_L(t) = \sum^{k}_{i=1} s_i \sigma\left(\iprod{w_i,x_0} - b_i + t\iprod{w_i,v}\right).
	\end{equation}

Note that as a univariate function, $F|_L(t)$ is a piecewise-linear function. We call a point $t \in \R$ a \emph{critical point} of $F|_L$ if the slope of the piecewise-linear function changes at $t$. Our starting point is that if we can identify the critical points of $F|_L$, then we can use estimates of the gradient of $F$ at those critical points to estimate the weights of the hidden units in $F$. More concretely, suppose we have identified a critical point $t_i$ of $F|_L$ for a line $L$. Now, back in the $d$-dimensional space, if we look at a sufficiently small neighborhood of the point $x' = x_0 + t_i \cdot v$, then the set of activated units around $x'$ changes by exactly one (in fact, by this reasoning, for a random $L$, $F|_L$ will have one unique critical point for each neuron in $F$, provided no two neurons are exact multiples of each other).  We can exploit this to retrieve the $w_i$'s and $b_i$'s by finite differencing: for each such $x'$, compute $(F(x' + \delta) - F(x'))/\|\delta\|$, for several sufficiently small perturbations $\delta \in \R^d$, and this will recover $s_i w_i$ (see Algorithm \ref{alg:getgrad} and \ref{lem:finite_diff} below for the details). We can recover $s_i b_i$ in a similar fashion (see Algorithm~\ref{alg:getbias}). We remark that the approach we have described appears to be quite similar to that of \cite{CarliniJM20,JagielskiCBKP20,MilliSDH19}.

However, the above description leaves out the following key interrelated challenges for getting provable guarantees for arbitrary networks:
\begin{enumerate}[leftmargin=*]
\item How do we identify the critical points of $F|_L$ for a given line $L$?
\item If we make no assumptions on how well-separated the weight vectors are, it is actually impossible to recover all of the weight vectors and biases in a bounded number of queries. Using the neurons that we do approximately recover, is it possible to piece them together to approximate $F$?
\end{enumerate}

A rigorous solution to these issues turns out to be quite tricky and technically involved. For instance, what if some critical points are arbitrarily close to each other (imagine the piecewise linear function being a tiny peak of exponentially small in $d$ width on the real line) on the line $F|_L$. In this case, identifying them would not be possible efficiently. We develop several structural results about cancellations in ReLU networks (culminating in Lemma~\ref{lem:ultimate_clump}) that show that such a situation only arises in degenerate cases where some sub-network of $F$ contains a cluster of many similar neurons (we formalize the notion of similarity in Definition~\ref{def:delta_alpha}). We show in Section \ref{sec:relucancel} that much smaller networks can approximate these sub-networks. On the other hand, for pairs of neurons which are sufficiently far from each other, the critical points on a random line $L$ will be well-separated.

With the structural results in hand, roughly speaking it suffices to identify one ``representative'' neuron for every cluster of similar neurons in $F$. To do this, we discretize $L$ into intervals of equal length and look at differences between gradients/biases of $F$ at the midpoints of these intervals. The key result leveraging our structural results is Theorem~\ref{lem:all_in_S} which shows that the set of all such differences comprises a rich enough collection of neurons that we can assemble some linear combination of them that approximates $F$. To find a suitable linear combination, we run linear regression on the features computed by these neurons (see Section~\ref{subsec:regression}).

\subsection{Comparison to Previous Approaches}
\label{sec:compare}

Our general approach of looking for critical lines along random restrictions of $F|_L$ is also the approach taken in the empirical works of \cite{CarliniJM20,JagielskiCBKP20} and in the theoretical work of \cite{MilliSDH19,daniely2021exact}, but we emphasize that there are a variety of subtleties that arise because we are interested in learning \emph{worst-case} networks from queries. In contrast, the empirical works consider trained networks arising in practice, while \cite{MilliSDH19} makes a strong assumption that the weight vectors $w_i$ are \emph{linearly independent} and that they are angularly separated. Note that such assumptions cannot hold in the practically relevant setting where $F$ is overparameterized. The recent work of \cite{daniely2021exact} likewise makes non-degeneracy assumptions that allow them to avoid dealing with closely spaced critical points, arguably the central technical challenge in the present work. Indeed, these sorts of assumptions are powerful enough that \cite{MilliSDH19,daniely2021exact} are able to \emph{exactly recover} the parameters of the network in a finite number of samples, whereas at the level of generality of the present work, this is clearly impossible.

To give a sense for the issues that these existing techniques run up against when it comes to worst-case networks, imagine that the one-dimensional piecewise linear function given by the restriction $F|_L$ were simply a ``bump,'' that is, it is zero over most of the line except in a small interval $[a,a+\delta]$. For instance, this would arise in the following example:


\begin{example}\label{example:bump}
    Consider the one-dimensional one hidden layer network $F:\R\to\R$ given by $F(x) = \sigma(x - a) +\sigma(x - a - \delta) - \sigma(2x - 2a - \delta)$. This function looks like the zero function except over the interval $[a,a+\delta]$, where it looks like a small ``bump.''
\end{example}

The proposal in the works mentioned above is to run a binary search to find a critical point. Namely, they initialize to some large enough interval $[-\tau,\tau]$ in $L$ and check whether the gradient at the left or right endpoint differs from the gradient at the midpoint $t_{\mathsf{mid}} = 0$. If the former, then they restrict the search to $[-\tau, t_{\mathsf{mid}}]$ and recurse until they end up with a sufficiently small interval, at which point they return the midpoint of that interval as an approximation to a critical point. The intuition is that $n$ steps of binary search suffice to identify a critical point up to $n$ bits of precision.\footnote{We remark that the refinement of binary search given in \cite{CarliniJM20,JagielskiCBKP20} can speed this up to $O(1)$ steps for networks that arise in practice, but the issue that we describe poses a challenge for their technique as well.}

It is clear, however, from the ``bump'' example that such a binary search procedure is doomed to fail: if the bump does not occur in the middle of the starting interval $[-\tau,\tau]$, then at the outset, we don't know which half to recurse on because the gradients at the endpoints and the gradient at the midpoint are all zero! Indeed, to locate the bump, it is clear that we need to query a number of points which is at least inverse in the width of the bump to even find an input to the network with nonzero output. 

We remark that this is in stark contrast to related active learning settings where the ability to query the function can sometimes yield \emph{exponential savings} on the dependence on $1/\epsilon$ (see e.g. \cite{hanneke2009theoretical}). Indeed, if we took the width of the bump to be of size $\poly(\epsilon)$ to make it $\epsilon$-far in square distance from the zero function, we would still need $\poly(1/\epsilon)$, rather than $\log(1/\epsilon)$, queries to locate the bump and distinguish the funcction from the zero function.

The core reason binary search fails for arbitrary networks is that there can be many disjoint linear pieces of $F|_L$ which all have the same slope. Indeed, because the gradient of $F$ at a given point is some linear combination of the weight vectors, if $F$ is overparameterized so that there are many linear dependencies among the weight vectors, it is not hard to design examples like Example~\ref{example:bump} where there may be many repeated slopes on any given random restriction of $F$ to a line.

Apart from these technical issues that arise in the worst-case setting we consider and not in previous empirical or theoretical works on model extraction, we also emphasize that our results are the first theoretical guarantees for learning \emph{general} one hidden-layer ReLU networks with bias terms in any learning model, including PAC learning from samples. As will quickly become evident in the sequel, biases introduce many technical hurdles in their own right. To our knowledge, the only other theoretical guarantees for learning networks with biases are the work of \cite{janzamin2015beating} which considers certain activations with nonzero second/third-order derivatives, precluding the ReLU, and the recent work of \cite{awasthi2021efficient} which considers ReLU networks whose weight vectors are well-conditioned.

\begin{remark}
    As discussed above, the ``bump'' construction in Example~\ref{example:bump} shows that unlike in related active learning contexts, $\poly(1/\epsilon)$ dependence is necessary in our setting. In fact this example also tells us another somewhat counterintuitive fact about our setting. Naively, one might hope to upgrade Theorem~\ref{thm:main} by replacing the dependence on the scaling parameters $R,B$ with one solely on the $L_2$ norm of the function. To see why this is impossible, consider scaling the function in Example~\ref{example:bump} by a factor of $\delta^{-3/2}$ to have unit $L_2$ norm. To learn $F$, we have to figure out the location of the bump, but this requires $\Omega(1/\delta)$ queries, and $\delta$ can be taken to be arbitrarily small.
\end{remark}

\paragraph{Limitations and Societal Impact} The line of work on learning neural networks from black-box access does pose a risk, for instance, that proprietary models offered via publicly-hosted APIs may be stolen by attackers who are only able to interact with the models through queries. The attackers might then use the extracted parameters of the model to learn sensitive information about the data the model was trained on, or perhaps to construct adversarial examples. That being said, understanding why this query learning problem can be easy for theorists is the first step towards building provable safeguards to ensure that it is hard for actual attackers. For instance, it is conceivable that one can prove information-theoretic or computational hardness for such problems if appropriate levels and kinds of noise are injected into the responses to the queries. Furthermore, query complexity lower bounds can inform how many accesses an API should allow any given user. 

\newcommand{\length}{r}
\newcommand{\grad}{\nabla}

\section{Preliminaries}

\paragraph{Notation} We let $\mathbb{S}^{d-1}$ denote the unit sphere in $\R^d$. Let $e_j$ denote the $j$-th standard basis vector in $\R^d$. Given vectors $u,v$, let $\angle(u,v)\triangleq \arccos\left(\frac{\iprod{u,v}}{\norm{u}\norm{v}}\right)$. Given a matrix $M$, let $\norm{M}_{\op}$ and $\norm{M}_F$ denote the operator and Frobenius norms respectively. Given a function $h$ which is square-integrable with respect to the Gaussian measure, we will use $\norm{h}$ to denote $\E[x\sim\calN(0,\Id)]{h(x)^2}^{1/2}$. Given a collection of indices $S\subseteq\Z$, we say that $i,j\in S$ are \emph{neighboring} if there does not exist $i < \ell < j$ for which $\ell \in S$.

The following elementary fact will be useful:
\begin{fact}\label{fact:sin_additive}
	$\abs{\sin(x+y)} = \abs{\sin(x)\cos(y) + \sin(y)\cos(x)} \le \abs{\sin(x)} + \abs{\sin(y)}$ for any $x,y\in\R$
\end{fact}


\subsection{Neural Networks, Restrictions, and Critical Points}

\begin{definition}\label{def:neuron}
    A \emph{neuron} is a pair $(v,b)$ where $v\in\R^d$ and $b\in\R$; it corresponds to the function $x\mapsto \sigma(\iprod{v,x}-b)$, which we sometimes denote by $\sigma(\iprod{v,\cdot} - b)$.
\end{definition}

As mentioned in the overview, we will be taking random restrictions of the underlying network $F$, for which we use the following notation:

\begin{definition}
	Given a line $L\subset\R^d$ parametrized by $L = \brc{x_0 + t\cdot v}_{t\in\R}$, and a function $F:\R^d\to\R$, define the \emph{restriction of $F$ to $L$} by $F|_L(t)\triangleq F(x_0 + t\cdot v)$. 
\end{definition}

\begin{definition}
    Given a line $L\subset\R^d$ and a restriction $F|_L$ of a piecewise linear function $F:\R^d\to\R$ to that line, the \emph{critical points} of $F|_L$ are the points $t\in\R$ at which the slope of $F|_L$ changes.
\end{definition}

\subsection{Concentration and Anti-Concentration} 

We will need the following standard tail bounds and anti-concentration bounds:
	
\begin{fact}[Concentration of norm of Gaussian vector]\label{fact:gaussian_conc}
	Given Gaussian vector $h\sim\calN(0,\Sig)$, $\Pr*{\norm{h} \ge O(\norm{\Sig^{1/2}}_{\op}(\sqrt{r} + \sqrt{\log(1/\delta)}))} \le \delta$, where $r$ is the rank of $\Sig$.
\end{fact}

\begin{fact}[Uniform bound on entries of Gaussian vector]\label{fact:gaussian_max_conc}
	For covariance matrix $\Sig\in\R^{m\times m}$, given $h\sim\calN(0,\Sig)$ we have that $\abs{h_i} \le O\left(\sqrt{\Sig_{i,i}}\sqrt{\log(m/\delta)}\right)$ for all $i\in[m]$ with probability at least $1 - \delta$.
\end{fact}

\begin{proof}
	For every $i\in[m]$, $h_i\sim\calN(0,\Sig_{i,i})$, so $\abs{h_i} \le O(\Sig_{i,i}^{1/2}\sqrt{\log(m/\delta)})$ with probability at least $1 - \delta/m$, from which the claim follows by a union bound and the fact that the largest diagonal entry of a psd matrix is the largest entry of that matrix.
\end{proof}

\begin{fact}[Carbery-Wright \cite{CW01}]\label{fact:carberywright}
	There is an absolute constant $C > 0$ such that for any $\nu > 0$ and quadratic polynomial $p:\R^d\to\R$, $\Pr[g\sim\calN(0,\Id)]{\abs{p(g)} \le \nu\cdot \Var{p(g)}^{1/2}} \le C\sqrt{\nu}$.
\end{fact}

\begin{lemma}[Anti-concentration of norm of Gaussian vector]\label{lem:anticonc}
	There is an absolute constant $C > 0$ such that given any Gaussian vector $h\sim\calN(\mu,\Sig)$, $\Pr*{\norm{h}\ge \sqrt{\nu}\norm{\Sig}^{1/2}_F} \ge 1 - C\sqrt{\nu}$.
\end{lemma}

\begin{proof}
	Define the polynomial $p(g) \triangleq (g + \mu)^{\top}\Sig(g + \mu)$. Note that for $g\sim\calN(0,\Id)$, $p(g)$ is distributed as $\norm{h}^2$ for $h\sim\calN(\mu,\Sig)$. We have $\E[g\sim\calN(0,\Id)]{p(g)} = \Tr(\Sig) +\mu^{\top}\Sig\mu$, so \begin{align}
		\Var{p(g)} &= \E{(g^{\top}\Sig g + 2g^{\top}\Sig \mu - \Tr(\Sig))^2} \\
		&= \E{(g^{\top}\Sig g)^2} + \E{(2g^{\top}\Sig\mu - \Tr(\Sig))^2} + 2\E{(g^{\top}\Sig g)(2g^{\top}\Sig\mu - \Tr(\Sig))} \\
		&= \left(2\Tr(\Sig^2) + \Tr(\Sig)^2\right) + \left(4\Tr(\Sig\mu\mu^{\top}\Sig) + \Tr(\Sig)^2\right) - 2\Tr(\Sig)^2 \\
		&= 2\iprod*{\Sig^2, \Id + 2\mu\mu^{\top}} \ge \norm{\Sig}^2_{F},
	\end{align} so by Fact~\ref{fact:carberywright} we conclude that $\Pr{p(g) \le \nu\norm{\Sig}_F} \le C\sqrt{\nu}$.
\end{proof}

\begin{lemma}[Anti-concentration for random unit vectors]\label{lem:anticonc2}
	For random $v\in\S^{d-1}$, $\Pr*{\abs{v_1} < \frac{\delta}{2\sqrt{d} + O(\sqrt{\log(1/\delta)})}} \le \delta$.
\end{lemma}

\begin{proof}
	For $g\sim\calN(0,\Id)$, $g/\norm{g}$ is identical in distribution to $v$. $\norm{g} \le \sqrt{d} + O(\sqrt{\log(1/\delta)})$ with probability at least $1 - \delta/2$ for absolute constant $c > 0$, and furthermore $\Pr[\gamma\sim\calN(0,1)]{|g| > t} \ge 1 - t$ for any $t > 0$, from which the claim follows by a union bound.
\end{proof}

\section{ReLU Networks with Cancellations}\label{sec:relucancel}

In the following section we prove several general results about approximating one hidden-layer networks with many ``similar'' neurons by much smaller networks.

\subsection{Stability Bounds for ReLUs}

The main result of this subsection will be the following stability bound for (non-homogeneous) ReLUs with the same bias.

\begin{lemma}\label{lem:non_hom_relu}
	Fix any $\Delta < 1$. For orthogonal $v,v'\in\R^d$ for which $\norm{v - v'} \le \Delta\norm{v}$, and $b\in\R$, we have \begin{equation}
		\E{(\sigma(\iprod{v,x} - b) - \sigma(\iprod{v',x} - b))^2} \le O\left(\Delta^{2/5}\norm{v}^2\right)
	\end{equation}
\end{lemma}

To prove this, we will need to collect some standard facts about stability of homogeneous ReLUs and affine threshold functions, given in Fact~\ref{fact:relucor}, Lemma~\ref{lem:hom_relus}, Lemma~\ref{lem:sheppard}, and Lemma~\ref{lem:perturb_bias}.

The following formula is standard \cite{ChoS09}:
\begin{fact}\label{fact:relucor}
	$\E{\sigma(\iprod{v,x})\sigma(\iprod{v',x})} = \frac{1}{2\pi}\norm{v}\norm{v'}\left(\sin\angle(v,v') + (\pi - \angle(v,v'))\cos\angle(v,v')\right)$. For $\iprod{v,v'}\ge 0$, note that this is at least $\frac{1}{6}\norm{v}\norm{v'} + \frac{1}{3}\iprod{v,v'}$.
\end{fact}

As a consequence, we obtain the following stability result for homogeneous ReLUs:

\begin{lemma}\label{lem:hom_relus}
	For any $v,v'\in \R^d$ for which $\iprod{v,v'} \ge 0$,  we have \begin{equation}
		\E{(\sigma(\iprod{v,x}) - \sigma(\iprod{v',x}))^2} \le \frac{1}{2}\norm{v - v'}^2 + \frac{2}{3}\norm{v}\norm{v'}(1 - \cos\angle(v,v'))
	\end{equation}	
\end{lemma}

\begin{proof}
	We can expand the expectation and apply Fact~\ref{fact:relucor} to get \begin{align}
		\E{(\sigma(\iprod{v,x}) - \sigma(\iprod{v',x}))^2} &= \E{\sigma(\iprod{v,x})^2} + \E{\sigma(\iprod{v',x})^2} - 2\E{\sigma(\iprod{v,x})\sigma(\iprod{v',x})} \\
		&\le \frac{1}{2}\norm{v}^2 + \frac{1}{2}\norm{v'}^2 - 2\left(\frac{1}{6}\norm{v}\norm{v'} + \frac{1}{3}\iprod{v,v'}\right) \\
		&= \frac{1}{2}\norm{v - v'}^2 + \frac{2}{3}(\norm{v}\norm{v'} - \iprod{v,v'}) \\
		&= \frac{1}{2}\norm{v - v'}^2 + \frac{2}{3}\norm{v}\norm{v'}(1 - \cos\angle(v,v'))
	\end{align}
	as claimed.
\end{proof}

We will also need the following stability result for affine linear thresholds.

\begin{lemma}[Lemma 5.7 in \cite{chen2020learning}] \label{lem:sheppard}
	Given $v,v'\in\R^d$ and $b\in\R$, \begin{equation}
		\Pr*{\iprod{v,x}>b \wedge \iprod{v',x} \le b} \le O(\norm{v - v'}/b).
	\end{equation}
\end{lemma}

\begin{lemma}\label{lem:perturb_bias}
	For any $v\in\R^d$ and $b\le b'$, \begin{equation}
		\E*{(\sigma(\iprod{v,x} - b) - \sigma(\iprod{v,x} - b'))^2} \le (b'-b)^2
	\end{equation}
\end{lemma}

\begin{proof}
	Note that $\iprod{v,x}\sim\calN(0,\norm{v}^2)$, so it suffices to show that for the univariate function $f(z) \triangleq \sigma(z-b) - \sigma(z - b')$, $\E[z\sim\calN(0,\norm{v}^2)]{f(z)^2}\le (b'-b)^2$. Observe that $f(z) = b' - b$ for $z > b'$, $f(z) = 0$ for $z < b$, and $f(z) = z - b$ for $z\in [b,b']$. In particular, $\abs{f(z)} \le b' - b$, from which the claim follows.
\end{proof}

The following basic lemma giving $L_2$ bounds for Lipschitz functions which are bounded with high probability will be useful throughout.

\begin{lemma}\label{lem:helper}
	Let $\epsilon(x): \R^d\to\R_{\ge 0}$ be any square-integrable function with respect to the Gaussian measure. If $f:\R^d\to\R$ is an $L$-Lipschitz continuous piecewise linear function and satisfies $\Pr[x\sim\calN(0,\Id)]{\abs{f(x)} \le \epsilon(x)} \ge 1 - \zeta$ and $\abs{f(0)} \le M$, then $\E[x\sim\calN(0,\Id)]{f(x)^2} \le 2\zeta M^2 +L^2\zeta^{1/2}(d^2 + 2d) + \E{\epsilon(x)^4}^{1/2}$.
\end{lemma}

\begin{proof}
	Because $f$ is $L$-Lipschitz, $f(x)^2 \le (M + L\norm{x})^2 \le 2M^2 + L^2\norm{x}^2$. Then \begin{align}
		\E{f(x)^2} &\le \E{f(x)^2\bone{f(x) > \epsilon(x)}} + \E{\epsilon(x)^2\bone{f(x) \le \epsilon(x)}} \\
		&\le 2\zeta M^2 + L^2\E{\norm{x}^2\bone{f(x) > \rho\norm{x}}} + \E{\epsilon(x)^4}^{1/2}(1-\zeta)^{1/2} \\
		&\le 2\zeta M^2 + L^2\zeta^{1/2}\E{\norm{x}^4} + \E{\epsilon(x)^4}^{1/2}(1-\zeta)^{1/2}\\
		&= 2\zeta M^2 + 3L^2\zeta^{1/2}d^2 + \E{\epsilon(x)^4}^{1/2},
	\end{align} as claimed.
\end{proof}

Putting all of these ingredients together, we can now complete the proof of the main Lemma~\ref{lem:non_hom_relu} of this subsection.

\begin{proof}
	Suppose $b \ge \Delta^{1/5}\norm{v}$. By Lemma~\ref{lem:sheppard}, $\sgn(\iprod{v,x} - b) \neq \sgn(\iprod{v',x} - b)$ with probability at most $O(\Delta\norm{v}/b)$. So with probability at least $1 - O(\Delta\norm{v}/b)$, the function $(\sigma(\iprod{v,x} - b) - \sigma(\iprod{v',x} - b)$ is at most $\iprod{v - v',x} \le \Delta\norm{v}\norm{x}$. Furthermore, this function is $L$-Lipschitz for $L = \norm{v} + \norm{v}' \le O(\norm{v})$. By Lemma~\ref{lem:helper} applied to the projection of $f$ to the two-dimensional subspace spanned by $v,v'$, \begin{equation}
		\E{(\sigma(\iprod{v,x} - b) - \sigma(\iprod{v',x} - b))^2} \lesssim \norm{v}^2\left(\sqrt{\Delta\norm{v}/b} + \Delta^2\right) \lesssim \Delta^{2/5}\norm{v}^2 .
	\end{equation}
	Now suppose $b < \Delta^{1/5}\norm{v}$. Then $\norm{\sigma(\iprod{v,\cdot} - b) - \sigma(\iprod{v,\cdot})}^2 \le \Delta^{2/5}\norm{v}^2$ and $\norm{\sigma(\iprod{v',\cdot} - b) - \sigma(\iprod{v',\cdot})}^2 \le \Delta^{2/5}\norm{v'}^2$. By triangle inequality, it suffices to bound $\norm{\sigma(\iprod{v,\cdot}) - \sigma(\iprod{v',\cdot})}^2$. By Lemma~\ref{lem:hom_relus}, we have \begin{equation}
		\norm{\sigma(\iprod{v,\cdot}) - \sigma(\iprod{v',\cdot})}^2 \lesssim \Delta^2\norm{v}^2 + \norm{v}^2 \cdot (1 - \cos\angle(v,v')) \lesssim \Delta^2\norm{v}^2,
	\end{equation} where the last step follows by the fact that $\norm{v - v'} \le \Delta\norm{v}$ implies that $\cos\angle(v,v') \ge \sqrt{1 - \Delta^2} \ge 1 - \Delta^2$.
\end{proof}

\subsection{\texorpdfstring{$(\Delta,\alpha)$}{(Delta,alpha)}-Closeness of Neurons}

We now formalize a notion of geodesic closeness between two neurons and record some useful properties. This notion is motivated by Lemma~\ref{lem:wellsep} in Section~\ref{sec:crit} where we study the critical points of random restrictions of one hidden-layer networks.

\begin{definition}\label{def:delta_alpha}
	Given $v,v'\in\R^d$ and $b,b'\in\R$, we say that $(v,b)$ and $(v',b')$ are $(\Delta,\alpha)$-close if the following two conditions are satisfied:
	\begin{enumerate}
		\item $\abs{\sin\angle(v,v')} \le \Delta$
		\item $\norm{b v' - b' v} \le \alpha\norm{v}\norm{v'}$.
	\end{enumerate} Note that this is a measure of angular closeness between $(v,b),(v',b')\in\R^{d+1}$. For instance, if $(v,b) = (\lambda v^*, \lambda b^*)$ and $(v',b') = (\lambda' v^*, \lambda' b^*)$ for some $(v^*,b^*)$, then $(v,b)$ and $(v',b')$ are $(0,0)$-close.
\end{definition}

We first collect some elementary consequences of closeness. The following intuitively says that if we scale two $(\Delta,\alpha)$-neurons to have similar norm, then their biases will be close.

\begin{lemma}\label{lem:delta_alpha_useful}
	If $(v,b)$ and $(v',b')$ are $(\Delta,\alpha)$-close, and $v = \gamma v' + v^{\perp}$ for $v^{\perp}$ orthogonal to $v'$, then $|\gamma b' - b| \le \alpha\norm{v}$.
\end{lemma}

\begin{proof}
	We know that $\norm{b v' - b' v} \le \alpha\norm{v}\norm{v'}$. The left-hand side of this is $\norm{(b - \gamma b') v' - b'v^{\perp}} \ge \abs{b - \gamma b'}\norm{v'}$, where the inequality follows from orthogonality of $v,v'$. Therefore, $\abs{\gamma b' - b} \le \alpha\norm{v}$ as claimed.
\end{proof}

Note that when two neurons are $(\Delta,\alpha)$-close, their weight vectors are either extremely correlated or extremely anti-correlated. In fact, given a collection of neurons that are all pairwise close, they will exhibit the following ``polarization'' effect.

\begin{lemma}\label{lem:orientation}
	Suppose $\Delta < \sqrt{2}/2$. If $(v_1,b_1),\ldots(v_k,b_k)$ are all pairwise $(\Delta,\alpha)$-close for some $\alpha > 0$, then there is a partition $[k] = S_1\sqcup S_2$ for which $\iprod{v_i,v_j} \ge 0$ for any $i\in S_1,j\in S_1$ or $i\in S_2, j\in S_2$, and for which $\iprod{v_i,v_j} < 0$ for any $i\in S_1, j\in S_2$ or $i\in S_2, j\in S_1$.
\end{lemma}

\begin{proof}
	Let $S_1$ be the set of $i\in[k]$ for which $\iprod{v_i,v_1} \ge 0$, and let $S_2$ be the remaining indices. First consider any $i,j\in S_1$ and note that $\angle(v_i,v_j) \le \angle(v_i,v_1) + \angle(v_j,v_1) \le 2\arcsin\Delta$, and because $\iprod{v_i,v_1},\iprod{v_j,v_1} \ge 0$, this is less than $\pi/4$ for $\Delta < \sqrt{2}/2$. By the same reasoning, we can show that for any $i,j\in S_2$, $\angle(v_i,v_j) < \pi/2$ if $\Delta < \sqrt{2}/2$. Finally, consider $i\in S_1$ and $j\in S_2$. We have $\angle(v_i,v_j) \ge \angle(v_j,v_1) - \angle(v_i,v_1)$. If $\Delta < \sqrt{2}/2$, then $\angle(v_j,v_1) > 3\pi/4$ while $\angle(v_i,v_1) < \pi/4$, concluding the proof.
\end{proof}

In the rest of the paper we will take $\Delta$ to be small, so Lemma~\ref{lem:orientation} will always apply. As such, it will be useful to define the following terminology:

\begin{definition}\label{def:orientation}
	Given $(v_1,b_1),\ldots,(v_k,b_k)$ which are all pairwise-close, we will call the partition $S_1\sqcup S_2$ given in Lemma~\ref{lem:orientation} the \emph{orientation induced by $\brc{(v_i,b_i)}$}.
\end{definition}

We note that $(\Delta,\alpha)$-closeness satisfies triangle inequality.

\begin{lemma}
	If $(v_1,b_1)$ and $(v_2,b_2)$ are $(\Delta,\alpha)$-close, and $(v_2,b_2)$ and $(v_3,b_3)$ are $(\Delta',\alpha')$-close, then $(v_1,b_1)$ and $(v_3,b_3)$ are $(\Delta + \Delta', 2\alpha +2\alpha')$-close.
\end{lemma}

\begin{proof}
	As $\angle(v_1,v_3) \le \angle(v_1,v_2) + \angle(v_2,v_3)$, it is clear from Fact~\ref{fact:sin_additive} that $\abs{\sin\angle(v_1,v_3)} \le \Delta + \Delta'$.

	Now write the orthogonal decompositions $v_1 = \gamma_1 v_2 + v^{\perp}_1$ and $v_3 = \gamma_3 v_2 + v^{\perp}_3$, noting that $\gamma_1\norm{v_2} \le \norm{v_1}, \gamma_3\norm{v_2} \le \norm{v_3}$. We can write \begin{equation}
		b_1 v_3 - b_3 v_1 = (b_1\gamma_3 - b_3\gamma_1) v_2 + (b_1v^{\perp}_3 - b_3 v^{\perp}_1). \label{eq:13}
	\end{equation} We will handle these two terms separately. First note that $(\Delta,\alpha)$-closeness of $(v_1,b_1), \,(v_2,b_2)$ and Lemma~\ref{lem:delta_alpha_useful} imply $\abs{b_2\gamma_1 - b_1} \le \alpha\norm{v_1}$, so in particular $|b_2\gamma_1\gamma_3 - b_1\gamma_3| \le \alpha\gamma_3\norm{v_1}$. Similarly, $|b_2\gamma_1\gamma_3 - b_3\gamma_1| \le \alpha'\gamma_1\norm{v_3}$. This allows us to conclude by triangle inequality that \begin{equation}
		\abs{b_1\gamma_3 - b_3\gamma_1}\cdot \norm{v_2} \le \alpha\gamma_1\norm{v_3} + \alpha'\gamma_3\norm{v_1})\norm{v_2}\le (\alpha + \alpha')\norm{v_1}\norm{v_3}. \label{eq:bgb13}
	\end{equation} 
	It remains to handle the second term on the right-hand side of \eqref{eq:13}. Note that Lemma~\ref{lem:delta_alpha_useful} also tells us that \begin{equation}
		\norm{b_1v^{\perp}_3 - b_3v^{\perp}_1} \le \norm{b_2\gamma_1 v^{\perp}_3 - b_1 v^{\perp}_3} + \norm{b_2\gamma_3 v^{\perp}_1 - b_3 v^{\perp}_1} \le \alpha\norm{v_1}\norm{v^{\perp}_3} + \alpha'\norm{v_3}\norm{v^{\perp}_1} \le (\alpha + \alpha')\norm{v_1}\norm{v_3}, \label{eq:bv13}
	\end{equation} so by \eqref{eq:13}, \eqref{eq:bgb13}, and \eqref{eq:bv13}, $\norm{b_1v_3 - b_3v_1} \le 2(\alpha+\alpha')\norm{v_1}\norm{v_3}$.
\end{proof}

\subsection{Merging Neurons}

In this section we begin to apply the tools we have developed in the preceding sections to show our main results about approximating neural networks with many close neurons by smaller networks. The goal of this subsection is to prove that a one hidden-layer network where  all neurons are $(\Delta,\alpha)$-close to some neuron can be approximated by at most two neurons:

\begin{lemma}\label{lem:main_clump}
	Given $F(x) = \sum^k_{i = 1}s_i \sigma(\iprod{w_i,x} - b_i)$ for $s_i \in \brc{\pm 1}$ and $(v^*,b^*)\in\R^d\times\R$ for which $(w_i,b_i)$ is $(\Delta,\alpha)$-close to $(v^*,b^*)$ for all $i\in[k]$, there exist coefficients $a^+, a^- \in \R$ for which \begin{equation}
		\E[x\sim\calN(0,\Id)]*{\left(F(x) - a^+\sigma(\iprod{v^*,x} - b^*) - a^-\sigma(\iprod{-v^*,x} + b^*)\right)^2 } \le O(k^2(\Delta^{2/5} + \alpha^2))\norm{v^*}^2.\label{eq:existence}
	\end{equation} Furthermore, we have that \begin{equation}
		|a^+|\norm{v^*}, |a^-|\norm{v^*} \le \sum_i \norm{w_i} \qquad \text{and} \qquad |a^+ b^*|, |a^- b^*| \le \alpha\sum_i \norm{w_i} + \sum_i |b_i|. \label{eq:bounds}
	\end{equation}
\end{lemma}

Our starting point for showing this is the following lemma which states that given \emph{two} close neurons whose weight vectors are correlated, we can merge them into a single neuron while incurring small square loss.

\begin{lemma}\label{lem:relus}
	Let $0 < \Delta \le 1$. For $v_1,v_2,v\in\R^d$, suppose we have $v_1 = \gamma_1 v + v^{\perp}_1$ and $v_2 = \gamma_2 v + v^{\perp}_2$ for $1\ge \gamma_1\ge \gamma_2\ge 0$ and $v^{\perp}_1,v^{\perp}_2$ orthogonal to $v$. Suppose additionally that $(v_1,b_1)$ and $(v_2,b_2)$ are both $(\Delta,\alpha)$-close to $(v,b)$. For $s\in\brc{\pm 1}$, we have that \begin{equation}
		\E[x\sim\calN(0,\Id)]*{\left(\sigma(\iprod{v_1,x} - b_1) + s \sigma(\iprod{v_2,x} - b_2) - (\gamma_1 + s\gamma_2)\sigma(\iprod{v,x} - b)\right)^2} \le O\left(\Delta^{2/5} + \alpha^2\right)\norm{v}^2
	\end{equation}
\end{lemma}

\begin{proof}
	For $i = 1,2$, because $\abs{\sin\angle(v_i,v)} \le \Delta$, we find $\norm{v^{\perp}_i} \le \Delta\norm{v_i} \le O(\Delta \norm{v})$ for $\Delta$ sufficiently small. From Lemma~\ref{lem:non_hom_relu} we have $\norm{\sigma(\iprod{v_i,\cdot} - b_i) - \sigma(\iprod{\gamma_i v,\cdot} - b_i)} \le O(\Delta^{1/5}\norm{v})$. Note that \begin{equation}
		(\gamma_i b - b_i)\norm{v}^2 = b\iprod{v,v_i} - b_i\norm{v}^2 \le \norm{v}\norm{bv_i - b_i v} \le \alpha\norm{v}^2\norm{v_i},
	\end{equation} i.e. $\gamma_i b - b_i \le \alpha\norm{v_i}$. So by Lemma~\ref{lem:perturb_bias}, $\norm{\sigma(\iprod{\gamma_i v,\cdot} - b_i) - \sigma(\iprod{\gamma_i v,\cdot} - \gamma_i b)} \le \alpha\norm{v_i}$. The lemma follows by triangle inequality and the fact that $\norm{v_i} \le \norm{v}\sqrt{1 + \Delta^2} \le 2\norm{v}$.
\end{proof}


Lemma~\ref{lem:relus} suggests the following binary operation.

\begin{definition}
	Fix a vector $v^*\in\R^d$. Consider the set of all triples $(s,v,b)$ for which $s\in\brc{\pm 1}$, $b\in\R$, and $v$ satisfies $0 \le \iprod{v,v^*} \le \norm{v^*}^2$. Define the binary operator $\odot_{v^*}$ as follows. Suppose $v_1 = \gamma_1 v + v^{\perp}_1$ and $v_2 = \gamma_2 v + v^{\perp}_2$ as in Lemma~\ref{lem:relus}, and define $\gamma = \abs{s_1\gamma_1 + s_2\gamma_2}$. Then \begin{equation}
		(s_1,v_1,b_1) \odot_{v^*} (s_2,v_2,b_2) = (s_i, \gamma v, \gamma b) \ \ \text{for} \ i = \arg\max_j \gamma_j
	\end{equation} Note that $s_i$ corresponds to the sign of $s_1\gamma_1 + s_2\gamma_2$, and $s_i\gamma = s_1\gamma_1 + s_2\gamma_2$.
\end{definition}

In this notation we can restate Lemma~\ref{lem:relus} as follows:

\begin{lemma}
 	For $v_1,v_2,b_1,b_2,v$, satisfying the conditions of Lemma~\ref{lem:relus}, if we define the tuple $(s',v',b')$ by $(s',v',b') = (s_1,v_1,b_1)\odot_{v}(s_2,v_2,b_2)$ we have that \begin{equation}
 		\E[x\sim\calN(0,\Id)]*{\left(s_1\sigma(\iprod{v_1,x} - b_1) + s_2\sigma(\iprod{v_2,x} - b_2) - s'\sigma(\iprod{v',x} - b')\right)^2} \le O\left(\Delta^{2/5} + \alpha^2\right)\norm{v}^2
 	\end{equation}
\end{lemma} 

It will be useful to record some basic properties of this binary operation:

\begin{fact}
	$\odot_{v^*}$ is associative and commutative. Moreover, if $(s_1,v_1,b_1)\odot_{v^*}\cdots\odot_{v^*} (s_m,v_m,b_m) = (s,\gamma v, \gamma b)$ for $s$ given by the sign of $\sum_i s_i \gamma_i$, where $v_i = \gamma_i v^* + v^{\perp}_i$ for $v^{\perp}_i$ orthogonal to $v^*$, then $s$ is the sign of $\sum s_i \gamma_i$, and $s \gamma = \sum s_i \gamma_i$.
\end{fact}

\begin{proof}
	That $\odot_{v^*}$ is commutative is evident from the definition. For associativity, consider $(s_1,v_1,b_1)$, $(s_2,v_2,b_2)$, $(s_3,v_3,b_3)$.
	Recall that if $(s_1,v_1,b_1)\odot_{v^*} (s_2,v_2,b_2) = (s_i,\gamma_{12} v^*,\gamma_{12} b)$ for $\gamma_{12} = \abs{s_1\gamma_1 + s_2\gamma_2}$, then $s_i$ corresponds to the sign of $s_1\gamma_1 + s_2\gamma_2$, so $s_i \gamma = s_1\gamma_1 + s_2\gamma_2$. We conclude that \begin{equation}
		((s_1,v_1,b_1)\odot_{v^*}(s_2,v_2,b_2))\odot_{v^*} (s_3,v_3,b_3) = (s_i,\gamma_{12}v^*,\gamma_{12}b)\odot_{v^*} (s_3,v_3,b_3) = (s_{i'}, \gamma_{123}v^*, \gamma_{123}b)
	\end{equation} for $\gamma_{123} = \abs{s_1\gamma_1 + s_2\gamma_2 + s_3\gamma_3}$ and $s_{i'}$ corresponding to the sign of $s_1\gamma_1 + s_2\gamma_2 + s_3\gamma_3$. It is therefore evident that $\odot_{v^*}$ is associative. The last part of the claim follows by induction.
\end{proof}

We show that merging many neurons which are all $(\Delta,\alpha)$ close to some given neuron $\sigma(\iprod{v^*,\cdot} - b^*)$ results in a neuron which is also close to $\sigma(\iprod{v^*,\cdot} - b^*)$.

\begin{lemma}\label{lem:combinestar}
	Let $m > 1$. Given $v_1,\ldots,v_m,v^*\in\R^d$ and $b_1,\ldots,b_m,b^*$ for which every $(v_i,b_i)$ is $(\Delta,\alpha)$-close to $(v^*,b^*)$ and satisfies $\iprod{v_i,v^*} \ge 0$, we have that for \begin{equation}
		(s,v,b) \triangleq (s_1,v_1,b_1) \odot_{v^*} \cdots \odot_{v^*} (s_m,v_m,b_m),
	\end{equation} $(v,b)$ is $(0,0)$-close to $(v^*,b^*)$ and satisfies $\iprod{v,v^*} \ge 0$. Furthermore, $\norm{v} \le \sum_i \norm{v_i}$ and $|b| \le \alpha\sum_i \norm{v_i} + \sum_i |b_i|$.
\end{lemma}

\begin{proof}
	Suppose first that $m = 2$. As usual, let $\Pi_{v^*}v_i = \gamma_i v^*$. Recall that $v = \gamma v^*$ and $b = \gamma b^*$ for $\gamma = \abs{s_1\gamma_1 + s_2\gamma_2}$. As a result, we clearly have that $\iprod{v,v^*} \ge 0$. Furthermore, \begin{equation}
		\norm{b v^* - b^* v} = \norm{\gamma b^* v^* - \gamma b^* v^*} = 0.
	\end{equation} The first part of the claim then follows by induction. For the norm bound, note that $\norm{v} = \abs*{\sum_i \gamma_i}\cdot \norm{v^*} \le \sum_i \norm{v_i}$. For the bound on $|b|$, recall from Lemma~\ref{lem:delta_alpha_useful} that for every $i$, $\norm{\gamma_i b^* - b_i} \le \alpha\norm{v_i}$. So $|b| = \abs*{\sum_i \gamma_i b^*} \le \sum_i (|b_i| + \alpha\norm{v_i})$ as claimed.
\end{proof}

Putting everything from this subsection together, we are now ready to prove Lemma~\ref{lem:main_clump}:

\begin{proof}[Proof of Lemma~\ref{lem:main_clump}]
	Denote $\odot_{v^*}$ by $\odot$. Let $S^+$ denote the set of $i\in[k]$ for which $\iprod{v^*,v_i} \ge 0$, and let $S^-$ denote the remaining indices $i\in[k]$. Define $F^+(x) \triangleq \sum_{i\in S^+} \sigma(\iprod{w_i,x} - b_i)$ and $F^-(x) \triangleq \sum_{i\in S^-} \sigma(\iprod{w_i,x} - b_i)$. By Lemma~\ref{lem:relus}, Lemma~\ref{lem:combinestar}, and triangle inequality, we have that for $(s^+,w^+,b^+) \triangleq \bigodot_{i\in S^+} (s_i,w_i,b_i)$ and $(s^-,w^-,b^-) \triangleq \bigodot_{i\in S^-} (s_i,w_i,b_i)$, \begin{equation}
		\norm{F^+ - s^+\sigma(\iprod{w^+,\cdot} - b^+)}^2, \norm{F^- - s^-\sigma(\iprod{w^-,\cdot} - b^-)}^2 \le O(k^2(\Delta^{2/5} + \alpha^2))\norm{v^*}^2.
	\end{equation} Recalling that $(w^+,b^+) = (\gamma^+ v^*, \gamma^+ b^*)$ and $(w^-,b^-) = (\gamma^- v^*, \gamma^- b^*)$, we conclude the proof of \eqref{eq:existence} with one more application of triangle inequality. For the bounds in \eqref{eq:bounds}, we simply apply the last part of Lemma~\ref{lem:combinestar}.
\end{proof}

\subsection{Constructing a Close Neuron}

Note that Lemma~\ref{lem:main_clump} requires the existence of a neuron $(v^*,b^*)$ which is close to all neurons $\brc{(v_i,b_i)}$. In our algorithm, we will not have access to $(v^*,b^*)$ but rather to some linear combination of the neurons $\brc{(v_i,b_i)}$. We first show that provided this linear combination is not too small in norm, it will also be close to all the neurons $\brc{(v_i,b_i)}$.

\begin{lemma}\label{lem:main_clump2}
	Suppose we have vectors $v_1,\ldots,v_m,v^*\in\R^d$, biases $b_1,\ldots,b_m,b^*\in\R$ for which every $(v_i,b_i)$ is $(\Delta,\alpha)$-close to $(v^*,b^*)$. Then for any $s_1,\ldots,s_m\in\brc{\pm 1}$, if we define $v \triangleq \sum^m_{i=1} s_i v_i$ and $b \triangleq \sum^m_{i=1}s_i b_i$, then $(v,b)$ is $\left(\Delta m, \alpha\sum_i\norm{v_i}/\norm{v}\right)$-close to $(v^*,b^*)$.
\end{lemma}

\begin{proof}
	Note that $\angle(\sum_i s_i v_i, v) \le \sum_i \angle(v_i,v)$. By Fact~\ref{fact:sin_additive}, we have that $\sin\angle(\sum_is_i v_i, v) \le \Delta m$.

	The lemma then follows from noting that \begin{equation}
		\norm*{b^* v - b v^*} = \norm*{\sum_i s_i (b_i v^*- b^* v_i)} \le \alpha \norm{v^*}\cdot \sum_i \norm{v_i} = \alpha\norm{v}\norm{v^*}\cdot \sum_i\norm{v_i}/\norm{v}.
	\end{equation}
\end{proof}

\subsection{A Corner Case}

This presents an issue: what if the linear combination of neurons that we get access to in our eventual algorithm has small norm, in which case Lemma~\ref{lem:main_clump2} is not helpful? It turns out this linear combination takes a very specific form (see the vector in \eqref{eq:lincombo}), and we argue in this section that if it is indeed small, then the underlying network we are trying to approximate will be close to a linear function! The main result of this subsection is to show:

\begin{lemma}\label{lem:corner_case}
	Suppose $(v_1,b_1),\ldots,(v_m,b_m)$ are pairwise $(\Delta,\alpha)$-close, and let $[m] = S_1\sqcup S_2$ denote the orientation induced by them (see Definition~\ref{def:orientation}). If signs $s_1,\ldots,s_m\in\brc{\pm 1}$ satisfy \begin{equation}
		\norm{\sum_{i\in S_1} s_i v_i - \sum_{i\in S_2}s_i v_i} \le (\Delta R)^{2/9}, \label{eq:lincombo}
	\end{equation} then for the network $F(x) \triangleq \sum_i s_i \sigma(\iprod{v_i,x} - b_i)$, there exists an affine linear function $\ell(x): \R^d\to\R$ for which \begin{equation}
		\E[x\sim\calN(0,\Id)]*{\left(F(x) - \ell(x)\right)^2} \le \poly(k,R,B)\cdot(\alpha^{1/2} + \Delta^{2/9}) \label{eq:approx_linear}
	\end{equation} 
	where $B \triangleq \max_i\norm{b_i}$ and $R\triangleq \max_i\norm{v_i}$, and $\ell \triangleq \iprod{w^*,\cdot} - b^*$ satisfying \begin{equation}
			\norm{w^*} \le \sum_i \norm{v_i} \qquad \text{and} \qquad |b^*| \le \sum_i \norm{b_i}.
		\label{eq:otherbounds}
	\end{equation}
\end{lemma}

Before proceeding to the proof, we will need the following stability result for affine linear threshold functions with possibly different thresholds.

\begin{lemma}\label{lem:closeness_sheppard}
	Suppose $(v,b)$ and $(v',b')$ are $(\Delta,\alpha)$-close and $\norm{v} \ge \norm{v'}$. If $\iprod{v,v'} \ge 0$ then \begin{equation}
		\Pr*{\iprod{v,x} > b \wedge \iprod{v',x} < b'} \le O\left(\alpha + \sqrt{\Delta\norm{v}/\norm{v'}}\right). \label{eq:close_sheppard}
	\end{equation} Otherwise, if $\iprod{v,v'} < 0$, then \begin{equation}
		\Pr*{\iprod{v,x} > b \wedge \iprod{v',x} > b'} \le O\left(\alpha + \sqrt{\Delta\norm{v}/\norm{v'}}\right). \label{eq:close_sheppard2}
	\end{equation}
\end{lemma}

\begin{proof}
	Clearly it suffices to prove \eqref{eq:close_sheppard}. Suppose $\norm{v'} \le \norm{v}$ and write $v' = \gamma v + v^{\perp}$ for $v^{\perp}$ orthogonal to $v$. Note that $\norm{v^{\perp}} \le \Delta\norm{v'}$ and that $\norm{v^{\perp}} \le \gamma\norm{v}\cdot \tan\angle(v,v') \le O(\gamma\Delta\norm{v})$ for $\Delta$ sufficiently small.

	Note that \begin{equation}
		\Pr*{\sgn(\iprod{v',x} - \gamma b) \neq \sgn(\iprod{v',x} - b')} \le \Pr[g\sim\calN(0,\norm{v'}^2)]{g\in[\Min{\gamma b}{b'},\Max{\gamma b}{b'}]} \le \frac{\abs{b' - \gamma b}}{2\norm{v'}}. \label{eq:shift}
	\end{equation}
	Because $\norm{b v' - b' v} = \norm{(b\gamma - b')v + b v^{\perp}} \le \alpha\norm{v}\norm{v'}$, we have that $|b\gamma - b'| \le \alpha\norm{v'}$. We conclude that $\Pr*{\sgn(\iprod{v',x} - \gamma b) \neq \sgn(\iprod{v',x} - b')} \le \alpha/2$.

	So by a union bound it suffices to bound $\Pr{\iprod{\gamma v,x} > \gamma b \wedge \iprod{v',x} < \gamma b}$. By Lemma~\ref{lem:sheppard}, this is at most $\frac{\norm{v' - \gamma v}}{\gamma b} = \frac{1}{\gamma b}\norm{v^{\perp}} \le O(\frac{\Delta\norm{v}}{b})$.

	We can also bound this in a different way. By a similar calculation to \eqref{eq:shift}, we have $\Pr{\sgn(\iprod{\gamma v, x} - \gamma b) \neq \sgn(\iprod{\gamma v,x})} \le \frac{b}{2\norm{v}}$ and $\Pr{\sgn(\iprod{v',x} - b) \neq \sgn(\iprod{v',x})} \le \frac{b}{2\norm{v'}}$. And by Sheppard's formula, $\Pr{\sgn(\iprod{v,x})\neq \sgn(\iprod{v',x})} \le \frac{\angle(v,v')}{\pi} \le O(\Delta)$ for $\Delta$ sufficiently small.

	We conclude that \begin{equation}
		\Pr{\iprod{\gamma v,x} > \gamma b \wedge \iprod{v',x} < \gamma b} \lesssim \Min{\frac{\Delta\norm{v}}{b}}{\left(\frac{b}{\norm{v'}} + \Delta\right)} \lesssim \sqrt{\Delta\norm{v}/\norm{v'}},
	\end{equation} from which the claim follows.
\end{proof}

We can now prove Lemma~\ref{lem:corner_case}.

\begin{proof}[Proof of Lemma~\ref{lem:corner_case}]
	Define $\omega\triangleq \norm{\sum_{i\in S_1} s_i v_i - \sum_{i\in S_2}s_i v_i}$. Let $S_0\subseteq[m]$ denote the set of $i$ for which $\norm{v_i} \le (\Delta R)^{1/9}$. For $i\in S_0$, note that by Lipschitz-ness of the ReLU function, \begin{equation}
		\norm{\sigma(\iprod{v_i,\cdot} - b_i) - \sigma(-b_i)}^2 \le \norm{\iprod{v_i,\cdot}}^2 = \norm{v_i}^2 \le \Delta^{2/9}R^{2/9}.
	\end{equation} So by triangle inequality it suffices to show that $\sum_{i\not\in S_0} s_i \sigma(\iprod{v_i,x} - b_i)$ is well-approximated by some affine linear function. We will thus assume without loss of generality that $S_0 = \emptyset$.

	By Lemma~\ref{lem:closeness_sheppard} and a union bound over all pairs $i,j\in[m]$, we have that with probability at least $1 - O(m^2\alpha +m^2\Delta^{4/9}R^{4/9})$ over $x\sim\calN(0,\Id)$, $\sgn(\iprod{v_i,x} - b_i) = \sgn(\iprod{v_j,x} - b_j)$ is the same for all $i,j\in S_1$ and for all $i,j\in S_2$, and $\sgn(\iprod{v_i,x} - b_i) \neq \sgn(\iprod{v_j,x} - b_j)$ for all $i\in S_1, j\in S_2$. Let $\bone{x\in\calE}$ denote the indicator for this event. In other words, with high probability all of the neurons in $S_1$ are activated and none in $S_2$ are, or vice versa; denote these two events by $\calE_1$ and $\calE_2$ respectively.

	For $j = 1,2$, note that when $x\in\calE_j$, $F(x) = \iprod*{\sum_{i\in S_j} s_i v_i,x} - \sum_{i\in S_j} s_i b_i$. Define $\ell(x) = \iprod*{\sum_{i\in S_1} s_i v_i,x} - \sum_{i\in S_1} s_i b_i$. Obviously when $x\in S_1$, $F(x) = \ell(x)$. To handle $x\in S_2$, we need to bound $\delta \triangleq \abs*{\sum_{i\in S_1} s_i b_i - \sum_{i\in S_2} s_i b_i}$. Let $(v,b) = (v_1,b_1)$ and note that because $(v_i,b_i)$ is $(\Delta,\alpha)$-close to $(v,b)$ for all $i$, \begin{equation}
		\alpha\norm{v}\sum_i \norm{v_i} \ge \norm*{\left(\sum_{i\in S_1} s_i b_i - \sum_{i\in S_2} s_i b_i\right)v - b\left(\sum_{i\in S_1} s_i v_i - \sum_{i\in S_2} s_i v_i\right)} \ge \delta\norm{v} - |b|\omega.
	\end{equation} In particular, $\delta \le \alpha\sum_i \norm{v_i} +|b|\omega/\norm{v} \le \alpha R + B\omega/(\Delta R)^{1/9}$.

	We would like to apply Lemma~\ref{lem:helper} to $F(x) - \ell(x)$ (projected to the span of $\brc{v_i}$). In that lemma, we can take $\epsilon(x) \le \abs*{\iprod*{\sum_{i\in S_1} s_i v_i - \sum_{i\in S_2} s_i v_i,x}} + \delta$, for which we have $\E{\epsilon(x)^4}^{1/2} \le O(\delta^2 + \omega^2)$. Additionally we can naively bound $F(0) - \ell(0)\le 2\sum_i |b_i|$ and therefore take $M$ in that lemma to be $2\sum_i |b_i| \le 2mB$. In addition, we can take $\zeta = O(m^2\alpha +m^2\Delta^{4/9}R^{4/9})$, $L = 2mR$, and $d = ,$.

	We conclude that \begin{align}
		\E{(F(x) - \ell(x))^2} &= \E{(F(x) - \ell(x))^2 \bone{x\in \calE_2}} + \E{(F(x) - \ell(x))^2 \bone{x\not\in \calE}} \\
		&\lesssim (m\alpha^{1/2} + m\Delta^{2/9}R^{2/9})\cdot(m^2B^2 + m^4 R^2) + \alpha^2 R^2 + B^2\omega^2/(\Delta R)^{2/9} +\omega^2.
	\end{align}
	Recalling that we paid an additional $m^2(\Delta R)^{2/9}$ in square loss in reducing to the case where $S_0 = \emptyset$, we obtain the desired bound in \eqref{eq:approx_linear}. The bounds in \eqref{eq:otherbounds} follow immediately from the definition of $\ell$ above.
\end{proof}

\subsection{Putting Everything Together}

Putting Lemmas~\ref{lem:main_clump}, \ref{lem:main_clump2}, and \ref{lem:corner_case} together, we conclude that networks whose hidden units are pairwise $(\Delta,\alpha)$-close can either be approximated by a particular size-two network, or by \emph{some} affine linear function:

\begin{lemma}\label{lem:ultimate_clump}
	Suppose $(v_1,b_1),\ldots,(v_k,b_k)$ are pairwise $(\Delta,\alpha)$-close, and let $[k] = S_1\sqcup S_2$ denote the orientation induced by them (see Definition~\ref{def:orientation}). Define $B \triangleq \max_i\norm{b_i}$ and $R\triangleq \max_i\norm{v_i}$. Let $s_1,\ldots,s_m\in\brc{\pm 1}$. 

	Define $F(x) = \sum_i s_i \sigma(\iprod{v_i,x} - b_i)$, $v^* = \sum_{i\in S_1} s_i v_i - \sum_{i\in S_2} s_i v_i$, and $b^* = \sum_{i\in S_1} s_i b_i - \sum_{i\in S_2} s_i b_i$. At least one of the following holds:
	\begin{enumerate}
		\item There is an affine linear function $\ell:\R^d\to\R$ for which $\norm{F - \ell}^2 \le \poly(k,R,B)\cdot(\alpha^{1/2} + \Delta^{2/9})$.
		\item There exist coefficients $a^+,a^-\in\R$ for which $G(x) \triangleq a^+\sigma(\iprod{v^*,x} - b^*) - a^-\sigma(\iprod{-v^*,x} + b^*)$ satisfies $\norm{F - G}^2 \le \poly(k,R,B)\cdot(\Delta^{2/5} + \alpha^2\Delta^{-4/9})$.
	\end{enumerate}
\end{lemma}

\begin{proof}
	By assumption, every $(v_i,b_i)$ is $(\Delta,\alpha)$-close to $(v_1,b_1)$. By Lemma~\ref{lem:main_clump2} we get that for $(v^*,b^*)$ defined in the lemma statement, $(v_1,b_1)$ is $(\Delta k, \alpha m R/\norm{v^*})$-close to $(v^*,b^*)$.

	If $\norm{v^*} \ge (\Delta R)^{2/9}$, then we conclude that $(v_1,b_1)$ is $(\Delta k, \alpha m \Delta^{-2/9} R^{7/9})$-close to $(v^*,b^*)$, and by Lemma~\ref{lem:main_clump} we find that there is a choice of $a^+,a^-$ for which the function $G$ defined in the lemma statement satisfies $\norm{F - G}^2 \le O(k^4R^2(\Delta^{2/5}k^{2/5} +\alpha^2 m^2 \Delta^{-4/9} R^{14/9}))$ (note that we used $\norm{v^*} \le \sum_i \norm{v_i} \le kR$).

	If $\norm{v^*} \le (\Delta R)^{2/9}$, then by Lemma~\ref{lem:corner_case} we find that there is an affine linear $\ell$ for which $\norm{F - \ell}^2 \le \poly(k,R,B) \cdot (\alpha^{1/2} +\Delta^{2/9})$.
\end{proof}

\section{Algorithm for Learning from Queries}
\label{sec:apply}

In this section we give our algorithm for learning neural networks from queries. Throughout, we will suppose we have black-box query access to some unknown one-hidden layer neural network \begin{equation}
	F(x) \triangleq \sum^{k}_{i = 1} s_i\sigma(\iprod{w_i,x} - b_i), \label{eq:Fdef}
\end{equation}
where $s_i \in \brc{\pm 1}$, $w_i\in\R^d$, $b_i\in\R$. Define the quantities $\radius\triangleq \max_i \norm{w_i}$ and $\B\triangleq \max_i |b_i|$; our bounds will be polynomial in these quantities, among others.

In Section~\ref{sec:crit}, we give bounds on the separation among critical points of random restrictions of $F$. In Section~\ref{sec:linesearch} we prove our main \emph{existence theorem} showing that by carefully searching along a random restriction of $F$, we are able to recover a collection of neurons that can be combined to approximate $F$. In Section~\ref{subsec:oracles} we show how to implement certain key steps in {\sc GetNeurons} involving querying the gradient and bias of $F$ at certain points. Finally, in Section~\ref{subsec:regression} we show to find an appropriate combination of these neurons.

\subsection{Critical Points of One-Hidden Layer Networks}
\label{sec:crit}

In this section, we compute the critical points of restrictions of $F$ and argue that they are far apart along \emph{random restrictions} unless if the corresponding neurons were close to begin with (in the sense of Definition~\ref{def:delta_alpha}). 

First, we formalize the notion of a random restriction:

\begin{definition}
	A \emph{Gaussian line} $L$ is a random line in $\R^d$ formed as follows: sample $x_0\sim\calN(0,\Id)$ and Haar-random $v\in\S^{d-1}$ and form the line $L \triangleq \brc{x_0 + t\cdot v}_{t\in\R}$.
\end{definition}

Here we compute the critical points along a restriction of $F$.

\begin{proposition}\label{prop:crit_point_formula}
	Given a line $L = \brc{x_0 + t\cdot v}_{t\in\R}$, the restriction $F|_L(t)\triangleq F(x_0 + t\cdot v)$ is given by \begin{equation}
		F|_L(t) = \sum^{k}_{i=1} s_i \sigma\left(\iprod{w_i,x_0} - b_i + t\iprod{w_i,v}\right).
	\end{equation}
	This function has $k$ critical points, namely $t = -\frac{\iprod{w_i,x_0} - b_i}{\iprod{w_i,v}}$ for every $i\in[k]$.
\end{proposition}

\begin{proof}
	The critical points of $F|_L$ are precisely the points $t$ at which a neuron changes sign. So the critical point associated to the $i$-th neuron is the $t$ for which $\iprod{w_i,x_0} - b_i + t\iprod{w_i,v} = 0$, from which the claim follows.
\end{proof}

We can show that these critical points are not too large, unless the norm of the corresponding weight vector is small. The reason for the latter caveat is that, e.g., if one took the one-dimensional neuron $\sigma(\epsilon z - b)$ for $b$ fixed and $\epsilon\to0$, the $z$ at which it changes sign tends to $\infty$).

\begin{lemma}\label{lem:radius}
	With probability at least $1 - \delta$ over the randomness of Gaussian line $L$, we have that $\abs{t_i} \lesssim \frac{k(\sqrt{d} + \sqrt{\log(1/\delta)})}{\delta\norm{w_i}} + k\left(\sqrt{d} +\sqrt{\log(1/\delta)}\right)\sqrt{\log(k/\delta)}$ for every critical point $t_i$ of $F|_L$.
\end{lemma}

\begin{proof}
	By Lemma~\ref{lem:anticonc2}, with probability $1 - \delta$ we have that $\abs{\iprod{w_i,v}} \gtrsim \frac{\delta\norm{w_i}}{k(\sqrt{d} + \sqrt{\log(1/\delta)})}$ for all $i\in[k]$. Also note that $\abs{\iprod{w_i,x_0}} \le \norm{w_i}\cdot \sqrt{\log(k/\delta)}$ for all $i\in[k]$ by Fact~\ref{fact:gaussian_max_conc}. By Proposition~\ref{prop:crit_point_formula}, the critical point corresponding to the $i$-th hidden unit satisfies \begin{align}
		\abs{t} = \abs*{\frac{\iprod{w_i,x_0} - b_i}{\iprod{w_i,v}}} &\lesssim \frac{k(\sqrt{d} + \sqrt{\log(1/\delta)})}{\delta\norm{w_i}}\left(\B + \norm{w_i}\sqrt{\log(k/\delta)}\right) \\
		&\le \frac{k(\sqrt{d} + \sqrt{\log(1/\delta)})}{\delta\norm{w_i}} + k\left(\sqrt{d} +\sqrt{\log(1/\delta)}\right)\sqrt{\log(k/\delta)}.
	\end{align}
\end{proof}

Fix a separation parameter $\Delta > 0$ which we will tune in the sequel. We show that along Gaussian lines $L$, $F|_L$'s critical points are well-separated except for those corresponding to neurons which are $(\Delta,\alpha)$-close.

\begin{lemma}\label{lem:wellsep}
	There is an absolute constant $c > 0$ for which the following holds. Given Gaussian line $L$, with probability at least $1 - \delta$ we have: for any pair of $i,j$ for which $(w_i,b_i)$ and $(w_j,b_j)$ are not $(\Delta,c\Delta\sqrt{\log(k/\delta)})$-close, the corresponding critical points are at least $\Omega\left(\frac{\Delta\delta^2}{k^4\left(\sqrt{d} + \sqrt{\log(k/\delta)}\right)}\right)$-apart. 
\end{lemma}

\begin{proof}
	For every $i\in[k]$, let $t_i\triangleq -\frac{\iprod{w_i,x_0} - b_i}{\iprod{w_i,v}}$ denote the location of the critical point corresponding to neuron $i$. For any $i,j\in[k]$,
	\begin{align}
		\abs{t_j - t_i} &= \abs*{\frac{\iprod{w_j,v}(\iprod{w_i,x_0} - b_i) - \iprod{w_i,v}(\iprod{w_j,x_0} - b_j)}{\iprod{w_i,v}\iprod{w_j,v}}} \\\label{eq:tjminusti}
		&\ge \frac{\abs*{\iprod*{\left(\iprod{w_i,x_0}w_j - \iprod{w_j,x_0}w_i\right) - \left(b_i w_j - b_j w_i\right),v}}}{\norm{w_i}\norm{w_j}} \triangleq \abs{\iprod{z_{ij},v}}.
	\end{align}
	Note that $\left(\iprod{w_i,x_0}w_j - \iprod{w_j,x_0}w_i\right) - \left(b_i w_j - b_j w_i\right)$ is distributed as $\calN(\mu,\Sig)$ for $\mu = -b_iw_j + b_j w_i$ and $\Sig^{1/2} = w_jw_i^{\top} - w_iw_j^{\top}$. One can verify that \begin{equation}
		\norm{\Sig}^{1/2}_F = 2^{1/4}\left(\norm{w_i}^2\norm{w_j}^2 - \iprod{w_i,w_j}^2\right)^{1/2} = 2^{1/4}\norm{w_i}\norm{w_j} \abs{\sin\angle(w_i,w_j)} 
	\end{equation}

	For the first part of the lemma, suppose $\abs{\sin\angle(w_i,w_j)} \ge \Delta$ so that $\norm{\Sig}^{1/2}_F \ge \Omega(\Delta\norm{w_i}\norm{w_j})$. Then by Lemma~\ref{lem:anticonc} we conclude that $\norm{z_{ij}} \ge \Omega(\Delta \delta/k^2)$ with probability at least $1 - \delta/k^2$. Recall that $v$ is a random unit vector drawn independently of $x_0$, so the lemma follows by applying Lemma~\ref{lem:anticonc2} and a union bound over all pairs $i,j$.

	On the other hand, suppose $\abs{\sin\angle(w_i,w_j)} \le \Delta$ but $\norm{\mu} \ge c\Delta\sqrt{\log(k/\delta)}\norm{w_i}\norm{w_j}$ for $c > 0$ sufficiently large. Note that $\Sig$ has rank 2, so by Fact~\ref{fact:gaussian_conc}, the norm of a sample from $\calN(0,\Sig)$ has norm at most $O(\norm{\Sig^{1/2}}_{\op}(\sqrt{2} + \sqrt{\log(k/\delta)})) = O(\Delta\norm{w_i}\norm{w_j}\sqrt{\log(k/\delta)})$ with probability at least $1 - \delta/k^2$. So if we take $c$ large enough that this is at least $\Omega\left(\frac{\Delta\delta^2}{k^4\sqrt{d}}\right)$ less than $c\Delta\norm{w_i}\norm{w_j}\sqrt{\log(k/\delta)}$, we conclude that $\norm{z_{ij}} \ge \Omega(\Delta\delta/k^2)$ with probability at least $1 - \delta/k^2$.
\end{proof}

\subsection{Line Search and Existence Theorem}
\label{sec:linesearch}

At a high level, our algorithm works by searching along $F|_L$, partitioning $L$ into small intervals, and computing differences between the gradients/biases of $F$ at the midpoints of these intervals. The primary structural result we must show is that there exists enough information in this set of differences to reconstruct $F$ up to small error.

As we will be working with partitions of lines, it will be convenient to define the following notation:

\begin{definition}
	Given line $L\subset\R^d$ and finite interval $I\subseteq\R$ corresponding to a segment $\mathcal{I}\subset L$, let $\grad_L(I)$ denote the gradient of $F$ at the midpoint of $\mathcal{I}$. For $t_{\mathsf{mid}}\in\R$ the midpoint of $I$, define $b_L(I) \triangleq F|_L(t_{\mathsf{mid}}) - (F|_L)'(t_{\mathsf{mid}})\cdot t_{\mathsf{mid}}$. Intuitively, this is the ``$y$-intercept'' of the linear piece of $F|_L$ that contains $t_{\mathsf{mid}}$. When $L$ is clear from context, we will drop subscripts and denote these objects by $\grad(I)$ and $b(I)$.
\end{definition}

\begin{definition}\label{def:G}
	Given line $L\subset\R^d$ and length $s > 0$, let $\brc{t_i}$ denote the critical points of $F|_L$, and let $G^+_L(s)\subseteq[k]$ (resp. $G^-_L(s)$) denote the set of indices $a\le i\le b$ for which $t_{i+1} - t_i \ge s$ (resp. $t_i - t_{i - 1} \ge s$). Let $G^*_L(r)\triangleq G^+_L(r) \cap G^-_L(r)$. 
\end{definition}

The following observation motivates Definition~\ref{def:G}:

\begin{observation}\label{obs:G}
	Given line $L\subset\R^d$, let $\brc{t_i}$ denote the critical points of $F|_L$. Let $I_1,\ldots,I_m$ be a partition of some interval $I$ into pieces of length $\length$, and for $t_i\in I$ let $\ell(i)$ denote the index of the interval containing $I$.

	Then for any $i\in G^+_L(2\length)$ (resp. $i\in G^-_L(2\length)$), $I_{\ell(i)+1}$ is entirely contained within $[t_i,t_{i+1}]$ (resp. $I_{\ell(i) - 1}$ is entirely contained within $[t_{i-1},t_i]$). In particular, $I_{\ell(i)-1}$ and $I_{\ell(i)+1}$ are linear pieces of $F|_L$.
\end{observation}

The following is the main result of this section. At a high level, it says if we partition a random line in $\R^d$ into sufficiently small intervals and can compute the gradient of $F$ at the midpoint of each interval, then we can produce a collection of neurons which can be used to approximate $F$.

\begin{theorem}\label{lem:all_in_S}
	For any $\epsilon, \delta> 0$, define 
	\begin{equation}
		\length\triangleq O\left(\frac{\Delta\delta^2}{k^4\left(\sqrt{d} + \sqrt{\log(k/\delta)}\right)}\right) \label{eq:lengthdef}
	\end{equation} \begin{equation}
		\tau \triangleq k\length + \Theta\left(\frac{k(\sqrt{d} + \sqrt{\log(1/\delta)})}{\delta\norm{w_i}} + k\left(\sqrt{d} +\sqrt{\log(1/\delta)}\right)\sqrt{\log(k/\delta)}\right). \label{eq:taudef}
	\end{equation}
	Partition the interval $[-\tau,\tau]$ into intervals $I_1,\ldots,I_m$ of length $\length$.

	Let $L$ be a Gaussian line, and let $\calS$ denote the set of all $m(m-1)$ pairs $(w,b)$ obtained by taking distinct $i,j\in[k]$ and forming $(\grad_L(I_i) - \grad_L(I_j), b_L(I_i) - b_L(I_j))$. There exist $\brc{\pm 1}$-valued coefficients $\brc{a_{w,b}}_{(w,b)\in\calS}$, vector $w^*$, and $b^*\in\R$ for which \begin{equation}
		\norm*{F - \sum_{(w,b)\in\calS} a_{w,b} \cdot \sigma(\iprod{w,\cdot} - b) - \iprod{w^*,\cdot} - b^*} \le \epsilon + \mathfrak{P}_{k,R,B,\log(1/\delta)}\cdot \Delta^{2/9}. \label{eq:error_all_in_S}
	\end{equation} for $\mathfrak{P}_{k,R,B,\log(1/\delta)}$ some absolute constant that is polynomially large in $k,R,B,\log(1/\delta)$. Furthermore, we have that \begin{equation}
		\norm{a_{w,b}\cdot w} \le kR \qquad \text{and} \qquad |a_{w,b}\cdot b| \le c\Delta k^2 R\sqrt{\log(k/\delta)} + kB \label{eq:bounds2}
	\end{equation}
	\begin{equation}
		\norm{w^*} \le kR \qquad \text{and} \qquad |b^*| \le  kB \label{eq:bounds3}
	\end{equation}
\end{theorem}

\begin{proof}
	Condition on the outcomes of Lemma~\ref{lem:wellsep} and Lemma~\ref{lem:radius} holding for $L$. Let $t_1,\ldots,t_k$ denote the critical points associated to neurons $w_1,\ldots,w_k$, and for convenience we assume without loss of generality that $t_1\le\cdots\le t_k$. Let $a, b\in[k]$ denote the indices for which $\abs{t_i} \le \tau$ for $i\in[a,b]$. By Lemma~\ref{lem:radius} and the definition of $\tau$, we have that for $i\not\in[a,b]$, $\norm{w_i} \le \epsilon/k$.

	By Lipschitzness of the ReLU function, \begin{align}
		\norm*{\sum_{i\not\in[a,b]} s_i \sigma(\iprod{w_i,\cdot} - b_i) - \sum_{i\not\in[a,b]} s_i \sigma(-b_i)} &\le \sum_{i\not\in[a,b]}\norm*{\sigma(\iprod{w_i,\cdot} - b_i) - \sigma(-b_i)} \\
		&\le \sum_{i\not\in [a,b]} \norm{w_i} \le (b - a+1)\epsilon/k. \label{eq:ignore_far}
	\end{align}

	Next, we handle the critical points $i\in[a,b]$.
	Given critical point $t_i$, let $\ell(i)\in[m]$ denote the index for which $t_i \in I_{\ell(i)}$. For convenience, denote $G^+_L(2\length), G^-_L(2\length), G^*_L(2\length)$ by $G^+, G^-, G^*$.
	By Observation~\ref{obs:G}, we know that for $i\in G^*$, the linear piece of $F|_L$ immediately preceding critical point $t_i$ contains $I_{\ell(i)-1}$, and the one immediately proceeding $t_i$ contains $I_{\ell(i)+1}$. Therefore, $\grad(I_{\ell(i) + 1}) - \grad(I_{\ell(i) - 1})$ and $b(I_{\ell(i)-1}) - b(I_{\ell(i)+1})$ are equal to $w_i$ and $b_i$ up to a sign, so $\calS$ must contain the neurons $(w_i,b_i)$ and $(-w_i,-b_i)$.

	Now consider any neighboring $i_1 < i_2$ in $G^+\Delta G^-$ for which $i_2 - i_1 > 1$; note that the latter condition implies that $i_1 \in G^-\backslash G^+$ and $i_2\in G^+\backslash G^-$, or else we would have a violation of the fact that $i_1$ and $i_2$ are neighboring. Furthermore, because $i_1,i_2$ are neighboring, for all $i_1 \le i \le i_2$ we have that $t_{i+1} - t_i \le 2\length$. By taking $\Delta$ in (the contrapositive of) Lemma~\ref{lem:wellsep} to be $\Delta\cdot k$, we conclude that for any $i_1 \le i < j \le i_2$, $(w_i,b_i)$ and $(w_j,b_j)$ are $(\Delta k,c\Delta k\sqrt{\log(k/\delta)})$-close for all such $i$. 

	Let $\brc{i_1,\ldots,i_2} = S_1\sqcup S_2$ denote the orientation induced by $(w_{i_1},b_{i_1}),\ldots(w_{i_2},b_{i_2})$. We would like to apply Lemma~\ref{lem:ultimate_clump} to the subnetwork $\wt{F}(x) \triangleq \sum^{i_2}_{j=i_1} s_j \sigma(\iprod{w_j,x} - b_j)$. By another application of Observation~\ref{obs:G}, we know that $\grad(I_{\ell(i_2)}) - \grad(I_{\ell(i_1)})$ and $b(I_{\ell(i_1)}) - b(I_{\ell(i_2)})$ are, up to a common sign, precisely the vector $v^*$ and bias $b^*$ defined in Lemma~\ref{lem:ultimate_clump}, so we conclude that either there exists a network $G$ consisting of neurons $\sigma(\iprod{v^*,x} - b^*)$ and $\sigma(\iprod{-v^*,x} + b^*)$ for which $\norm{\wt{F} - G}^2 \le \poly(k,R,B)\cdot (\Delta^{2/5}k^{2/5} + c^2 \Delta^{14/9}k^2\log(k/\delta)) \le \poly(k,R,B)\Delta^{2/5}\log(1/\delta)$, or there is an affine linear function $\ell$ for which $\norm{\wt{F} - \ell}^2 \le \poly(k,R,B)\cdot(c^{1/2}\Delta^{1/2}\log(1/\delta)^{1/2} + \Delta^{2/9}) \le \poly(k,R,B)\cdot\Delta^{2/9}\log(1/\delta)^{1/2}$. Furthermore, the bounds in \eqref{eq:bounds2} and \eqref{eq:bounds3} follow from \eqref{eq:bounds} in Lemma~\ref{lem:main_clump} (for $\alpha = c\Delta k\sqrt{\log(k/\delta)}$) and \eqref{eq:otherbounds} in Lemma~\ref{lem:corner_case} respectively.

	We have accounted for all critical points, except in the case where the smallest index $a'$ in $G^-$ is not $a$, or the largest index $b'$ in $G^+$ is not $b$. In the former (resp. latter) case, note that $t_a \le\cdots \le t_{a' - 1} \le -\tau + k\length$,  (resp. $t_{b}\ge \cdots \ge t_{b'+1} \ge \tau - k\length$), so by Lemma~\ref{lem:radius}, this implies that $\norm{w_{a'-1}},\ldots,\norm{w_{a}} \le \epsilon/k$ (resp. $\norm{w_{b'+1}},\ldots,\norm{w_b} \le \epsilon/k$). By Lipschitzness of the ReLU function, we can approximate these neurons by constants at a total cost of at most $(a' - a + b - b')\epsilon/k$ in $L_2$ using the same reasoning as \eqref{eq:ignore_far}.
\end{proof}

\subsection{Gradient and Bias Oracles}
\label{subsec:oracles}

It remains to implement oracles to compute $b_L(I)$ and $\grad_L(I)$ for prescribed line $L$ and interval $I$. It is not clear how to do this for arbitrarily small intervals because for general networks there can be many arbitrarily close critical points, but we will only need to do so for certain ``nice'' $I$ as suggested by Theorem~\ref{lem:all_in_S}.

To that end, first note that it is straightforward to form the quantities $b_L(I)$ for intervals $I$ entirely contained within linear pieces of $F|_L$; we formalize this in Algorithm~\ref{alg:getbias}.

\begin{algorithm}
\DontPrintSemicolon
\caption{\textsc{GetBias}($L,I$)}
\label{alg:getbias}
	\KwIn{Line $L\subset\R^d$, interval $I = [a,b]\subset\R$}
	\KwOut{$b_L(I)$ if $I$ is entirely contained within a linear piece of $F|_L$}
		$t_{\mathsf{mid}}\gets$ midpoint of $I$.\;
		$y_0 \gets F|_L(t_{\mathsf{mid}})$.\;
		$s \gets \frac{F|_L(b) - F|_L(a)}{b - a}$.\;
		\Return{$y_0 - s\cdot t_{\mathsf{mid}}$}
\end{algorithm}

It remains to demonstrate how to construct $\grad_L(I)$. Intuitively one can accomplish this via ``finite differencing,'' i.e. the gradient of a piecewise linear function $F$ at a point $x$ can be computed from queries by computing $\frac{F(x+\delta) - F(x)}{\delta}$ several sufficiently small perturbations $\delta\in\R^d$ and solving the linear system. 

With \emph{a priori} precision estimates, we can similarly implement a gradient oracle, as formalized in Algorithm~\ref{alg:getgrad} and Lemma~\ref{lem:finite_diff}.

\begin{algorithm}
\DontPrintSemicolon
\caption{\textsc{GetGradient}($x,\alpha$)}
\label{alg:getgrad}
	\KwIn{$x\in\R^d$, $\alpha > 0$ for which \eqref{eq:sep} holds}
	\KwOut{$\nabla F(x)\in\R^d$}
		\For{$j\in[d]$}{
			Sample random unit vector $z_j\in\S^{d-1}$.\;
			$v_j \gets (F(x +\alpha z_j) - F(x))/\alpha$.\;
		}
		Let $w$ be the solution to the linear system $\brc{\iprod{w,z_j} = v_j}_{j\in[d]}$.\; \label{step:system}
		\Return{$w$}
\end{algorithm}

\begin{lemma}\label{lem:finite_diff}
	For any $\alpha > 0$ and any $x\in\R^d$ for which \begin{equation}
		\abs{\iprod{w_i,x} - b_i} \ge \alpha\norm{w_i} \ \ \forall \ i\in[k],\label{eq:sep}
	\end{equation}
	{\sc GetGradient}($x,\alpha$) makes $d$ queries to $F$ and outputs $\nabla F(x)$.
\end{lemma}

\begin{proof}
	For any $z\in\S^{d-1}$, note that \begin{equation}
		\iprod{w_i,x+\alpha z} - b_i = \left(\iprod{w_i,x} - b_i\right) + \alpha \iprod{w_i,z},
	\end{equation} and $\alpha\abs{\iprod{w_i,z}} \le \alpha\cdot\norm{w_i}$, so $\iprod{w_i,x + \alpha z} - b_i$ and $\iprod{w_i,x} + b_i$ have the same sign. As a result, if $S\subseteq[k]$ denotes the indices $i$ for which $\iprod{w_i,x} - b_i > 0$, then \begin{equation}
		\frac{F(x + \alpha z) - F(x)}{\alpha} = \iprod*{\sum_{i\in S} s_i w_i, z} = \iprod*{\nabla F(x), z}.
	\end{equation} If $\brc{z_1,\ldots,z_j}$ are a collection of Haar-random unit vectors, they are linearly independent almost surely, in which case the linear system in Step~\ref{step:system} of {\sc GetGradient} has a unique solution, namely $\nabla F(x)$.
\end{proof}

In order to use {\sc GetGradient} to construct the vectors $\grad_L(I)$, we require estimates for $\alpha$ in Lemma~\ref{lem:finite_diff}. In the following lemma we show that with high probability over the randomness of $L$, if an interval $I$ completely lies within a linear piece of $F|_L$, then we can bound how small we must take $\alpha$ to query the gradient of $F$ at the midpoint of that interval.

\begin{lemma}{4.10}\label{lem:set_alpha}
	Let $L$ be a Gaussian line. With probability at least $1 - \delta$ over the randomness of $L$, the following holds: in the partition $[-\tau,\tau] = I_1\cup\cdots\cup I_m$ in Theorem~\ref{lem:all_in_S}, for any $I_{\ell}$ which entirely lies within a linear piece of $F|_L$, {\sc GetGradient}($t_{\mathsf{mid}},\alpha$) correctly outputs $\nabla_L(I_{\ell})$, where $x_{\mathsf{mid}}$ is the midpoint of the interval $\calI_{\ell}\subset L$ that corresponds to interval $I_{\ell}\subset\R$ and $\alpha = \frac{\delta\cdot\length}{4k\sqrt{d} +O(k\sqrt{\log(k/\delta)})}$ (where $\length$ is defined in \eqref{eq:lengthdef}).
\end{lemma}

\begin{proof}
	Denote $L = \brc{x_0 + t\cdot v}_{t\in\R}$. Let $t_{\mathsf{mid}}\in\R$ denote the value corresponding to $x_{\mathsf{mid}}\in\R^d$ on the line $L$. By Lemma~\ref{lem:anticonc2} and a union bound over $[k]$, we have that \begin{equation}
		\abs{\iprod{w_i,v}} \ge \frac{\delta\norm{w_i}}{2k\sqrt{d} + O(k\sqrt{\log(k/\delta)})} \ \ \text{for all} \ i\in[k]
	\end{equation} with probability at least $1 - \delta$ over the randomness of $v\in\S^{d-1}$. Now take any interval $I_{\ell}$ which entirely lies within a linear piece of $F|_L$. Because $t_{\mathsf{mid}}$ is the midpoint of $I_{\ell}$, it is at least $\length/2$ away from any critical point of $F|_L$. In particular, $\abs{\iprod{w_i,x_{\mathsf{mid}}} - b} \ge (\length/2)\cdot \abs{\iprod{w_i,v}} \ge (\length/2)\cdot \frac{\delta\norm{w_i}}{2k\sqrt{d} +O(k\sqrt{\log(k/\delta)})}$, so we can take $\alpha = \frac{\delta\cdot\length}{4k\sqrt{d} +O(k\sqrt{\log(k/\delta)})}$ and invoke Lemma~\ref{lem:finite_diff}.
\end{proof}

Putting these ingredients together, we obtain the following algorithm, {\sc GetNeurons} for producing a collection of neurons that can be used to approximate $F$.

\begin{algorithm}
\DontPrintSemicolon
\caption{\textsc{GetNeurons}($\epsilon,\delta$)}
\label{alg:getneurons}
	\KwIn{Accuracy $\epsilon>0$, confidence $\delta>0$}
	\KwOut{List $\calS$ of pairs $(w,b)$ (see Theorem~\ref{lem:all_in_S} for guarantee)}
	$\calS\gets\emptyset$.\;
	Sample Gaussian line $L$.\;
	$\Delta\gets (\epsilon/\mathfrak{P}_{k,R,B,\log(1/\delta)})^{9/2}$. \tcp*{Theorem~\ref{lem:all_in_S}}
	$\alpha \gets \frac{\delta\cdot\length}{4k\sqrt{d} +O(k\sqrt{\log(k/\delta)})}$. \tcp*{Lemma~\ref{lem:set_alpha}}
	Define $\length,\tau$ according to \eqref{eq:lengthdef}, \eqref{eq:taudef}.\;
	Partition $[-\tau,\tau]$ into disjoint intervals $I_1,\ldots,I_m$ of length $\length$.\;
	\For{all $j\in [m]$}{
		$x_j\gets$ midpoint of the interval $\calI_j\subset L$ that corresponds to $I_j\subset\R$.\;
		$\wh{\nabla}_L(I_j)\gets$ {\sc GetGradient}($x_j, \alpha$).\;
		$\wh{b}_L(I_j)\gets$ {\sc GetBias}($L,I_j$).\;
	}
	\For{all pairs of distinct $i,j\in[m]$}{
		$(v_j,b_j) \gets (\wh{\nabla}_L(I_i) - \wh{\nabla}_L(I_j), \wh{b}_L(I_i) - \wh{b}_L(I_j))$.\;
		\If{$(v_j,b_j)$ satisfies the bounds in \eqref{eq:bounds2}}{
			Add $(v_j,b_j)$ to $\calS$.\; \label{step:keep}
		}
	}
	\Return{$\calS$}.\;
\end{algorithm}

We prove correctness of {\sc GetNeurons} in the following lemma:

\begin{lemma}{4.11}\label{lem:getneurons_guarantee}
	For any $\epsilon,\delta>0$, {\sc GetNeurons}($\epsilon,\delta$) makes $\poly(k,d,R,B,1/\epsilon,\log(1/\delta))$ queries and outputs a list $\calS$ of pairs $(w,b)$ for which there exist $\brc{\pm 1}$-valued coefficients $\brc{a_{w,b}}_{(w,b)\in\calS}$ as well as a vector $w^*$ and a scalar $b^*$ such that \begin{equation}
		\norm*{F - \iprod{w^*,\cdot} - b^* - \sum_{(w,b)\in\calS} a_{w,b} \cdot \sigma(\iprod{w,\cdot} - b)} \le \epsilon.
	\end{equation}
\end{lemma}

\begin{proof}
	By Lemma~\ref{lem:set_alpha}, the choice of $\alpha$ in {\sc GetNeurons} is sufficiently small that for $x_j$ the midpoint of any interval which is entirely contained within a linear piece of $F|_L$, {\sc GetGradient}$(x_j,\alpha)$ succeeds by Lemma~\ref{lem:finite_diff}. So the estimates $\wh{\nabla}$ and $\wh{b}$ are exactly correct for all intervals that are entirely contained within a linear piece of $F|_L$. By the proof of Theorem~\ref{lem:all_in_S}, these are the only intervals for which we need $\nabla_L(I)$ and $b_L(I)$ in order for $\calS$ to contain enough neurons to approximate $F$ by some linear combination to $L_2$ error $\epsilon$.
\end{proof}

\subsection{Linear Regression Over ReLU Features}
\label{subsec:regression}

It remains to show how to combine the neurons produced by {\sc GetNeurons} to obtain a good approximation to $F$. As Theorem~\ref{lem:all_in_S} already ensures that some linear combination of them suffices, we can simply draw many samples $(x,F(x))$ for $x\sim\calN(0,\Id)$, form the feature vectors computed by the neurons output by {\sc GetNeurons}, and run linear regression on these feature vectors.

Formally, let $\calS$ denote the set of pairs $(w,b)$ guaranteed by Theorem~\ref{lem:all_in_S}. We will denote the $w$'s by $\brc{\wh{w}_j}$ and the $b$'s by $\brc{\wh{b}_j}$. Consider the following distribution over feature vectors computed by the neurons in $\calS$:

\begin{definition}\label{def:D}
	Let $\calD'$ denote the distribution over $\R^{|\calS| +d+1}\times \R$ of pairs $(z,y)$ given by sampling $x\sim\calN(0,\Id)$ and forming the vector $z$ whose entries consist of all $\sigma(\iprod{\wh{w}_j,x} - \wh{b}_j)$ as well as the entries of $x$ and the entry 1, and taking $y$ to be $F(x)$ for the ground truth network $F$ defined in \eqref{eq:Fdef}.
	
	We will also need to define a truncated version of $\calD'$: let $\calD$ denote $\calD'$ conditioned on the norm of the $|\calS| +1$ to $|\calS|+d$-th coordinates having norm at most $M \triangleq \sqrt{d} +O(\sqrt{\log(1/\delta)}$, which happens with probability at least $1 - \delta$ over $\calD'$.
\end{definition}

Our algorithm will be to sample sufficiently many pairs $(z,y)$ from $\calD'$ (by querying $F$ on random Gaussian inputs) and run ordinary least squares. This is outlined in {\sc LearnFromQueries} below.

\begin{algorithm}
\DontPrintSemicolon
\caption{\textsc{LearnFromQueries}($\epsilon,\delta$)}
\label{alg:learnfromqueries}
	\KwIn{Accuracy $\epsilon>0$, confidence $\delta>0$}
	\KwOut{One hidden-layer network $\wt{F}:\R^d\to\R$ for which $\norm{F - \wt{F}} \le O(\epsilon)$}
		$\calS = \brc{(\wh{w}_j, \wh{b}_j)}\gets$ {\sc GetNeurons}($\epsilon,\delta$).\;
		Draw samples $(z_1,y_1),\ldots,(z_n,y_n)$ from $\calD$ \tcp*{Definition~\ref{def:D}}
		Let $\wt{v}$ be the solution to the least-squares problem \eqref{eq:leastsquares}. Let $\wt{b}$ denote the last entry of $\wt{v}$, and let $\wt{w}$ denote the vector given by the $d$ entries of $\wt{v}$ prior to the last. \;
		Form the network $\wt{F}(x) \triangleq \sum_j \wt{v}_j \sigma(\iprod{\wh{w}_j,x} - \wh{b}_j) + \iprod{\wt{w},\cdot} - \wt{b}$.\;
		\Return{F}.\;
\end{algorithm}

To show that regression-based algorithm successfully outputs a network that achieves low population loss with respect to $F$, we will use the following standard results on generalization.

\begin{theorem}\label{thm:gen}
	For $\calD$ a distribution over $\calX\times\calY$ and $\ell:\calY\times\calY\to\R$ a loss function that is $L$-Lipschitz in its first argument and uniformly bounded above by $c$. Let $\calF$ be a class of functions $\calX\to\calY$ such that for any $f\in\calF$ and pairs $(x_1,y_1),\ldots,(x_n,y_n)$ drawn independently from $\calD$, with probability at least $1 - \delta$, \begin{equation}
		\E[(x,y)\sim\calD]*{\ell(f(x),y)} \le \frac{1}{n}\sum_i \ell(f(x_i),y_i) + 4L\cdot \calR_n(\calF) + 2c\cdot \sqrt{\frac{\log(1/\delta)}{2n}},
	\end{equation} where $\calR_n(\calF)$ denotes the Rademacher complexity of $\calF$.
\end{theorem}

\begin{theorem}\label{thm:rademacher}
	If $\calX$ is the set of $x$ satisfying $\norm{x} \le X$, and $\calF$ is the set of linear functions $\iprod{w,\cdot}$ for $\norm{w} \le W$, then $\calR_n(\calF) \le XW/\sqrt{n}$.
\end{theorem}

As these apply to bounded loss functions and covariates, we must first pass from $\calD'$ to $\calD$ and quantify the error in going from one to the other:

\begin{lemma}\label{lem:compareDD'}
	For $f$ satisfying $\E[(z,y)\sim\calD']{(f(z) - y)^2} \le \epsilon^2$, we have \begin{equation}
		\abs*{\E[(z,y)\sim\calD']{(f(z) - y)^2} - \E[(x,y)\sim\calD]{(f(z) - y)^2}} \le O(\epsilon^2). \label{eq:compareDD'}
	\end{equation}
\end{lemma}

\begin{proof}
	Let $Z$ denote the probability that a random draw from $\calD'$ lies in the support of $\calD$ so that $Z \ge 1 - \delta$; denote this event by $\calE$. Then we can write $\E[(z,y)\sim\calD]{(f(z) - y)^2}$ as $\frac{1}{Z}\E[(z,y)\sim\calD']{(f(z) - y)^2 \cdot \bone{z\in\calE}}$ and rewrite the left-hand side of \eqref{eq:compareDD'} as \begin{equation}
		\abs*{\left(1 - \frac{1}{Z}\right)\cdot\E[(z,y)\sim\calD']*{(f(z) - y)^2 \cdot \bone{z\in\calE}} + \E[(z,y)\sim\calD']*{(f(z) - y)^2\cdot \bone{z\not\in\calE}}}.
	\end{equation}
	Note that $\abs{1 - 1/Z}\le 2\delta \le 1$ for $\delta$ sufficiently small, from which the claim follows.
\end{proof}

We are now ready to prove the main theorem of this section:

\begin{theorem}{4.12}\label{thm:linearregression}
    Let $\calS$ denote the list of pairs $(\wh{w}_j,\wh{b}_j)$ output by {\sc GetNeurons}($\epsilon,\delta$). Sample $(z_1,y_1),\ldots,(z_n,y_n)$ from $\calD$ for $n = \poly(k,R,B,1/\epsilon,d,\log(1/\delta))$. With probability at least $1 - O(\delta)$ over the randomness of {\sc GetNeurons} and the samples, the following holds. Define \begin{equation}
	    \wt{v} \triangleq \arg\min_{\norm{v} \le W} \sum^n_{i= 1}(\iprod{v,z_i} - y_i)^2, \ \text{for} \ W\triangleq \sqrt{\tau/\length} + k(R+B), \label{eq:leastsquares}
	\end{equation} let $\wt{b}$ denote the last entry of $\wt{v}$, and let $\wt{w}$ denote the vector given by the $d$ entries of $\wt{v}$ prior to the last. Then the one hidden-layer network $\wt{F}(x) \triangleq \sum_j \wt{v}_j \sigma(\iprod{\wh{w}_j,x} - \wh{b}_j)  + \iprod{\wt{w},\cdot} - \wt{b}$ satisfies $\norm{F - \wt{F}} \le O(\epsilon)$.
\end{theorem}

\begin{proof}
	Note that over the support of $\calD$ we have that the square loss $\ell: \calY\times\calY\to\R$ is uniformly bounded above by $(MkR + kB)^2$ and is $L = O(M\cdot k\cdot R +k\cdot B)$-Lipschitz. Finally, note that for $z$ in the support of $\calD$, \begin{multline}
	\norm{z}^2 =  1 + M^2 + 2M^2\sum_j (\norm{\wh{w}_j}^2 + \wh{b}_j^2) \\ \lesssim (M^2\tau/\length)\cdot(k^2R^2 + \Delta^2k^4R^2\log(k/\delta) +k^2B^2) \triangleq X^2. \label{eq:zbound}
	\end{multline} where $\tau,\length$ are defined in Theorem~\ref{lem:all_in_S} and we used \eqref{eq:bounds2} and Step~\ref{step:keep} in {\sc GetNeurons} to bound $\norm{\wh{w}_j}$ and $|\wh{b}_j|$.

	By the guarantee on {\sc GetNeurons} given by Lemma~\ref{lem:getneurons_guarantee}, we know that there is a vector $v^*\in\brc{\pm 1}^{|\calS|}\times B^d(kR)\times [-kB,kB]$ which achieves $\epsilon^2$ squared loss with respect to $\calD'$. Note that \begin{equation}
		\norm{v^*} \le |\calS|^{1/2} + k(R+B) = \sqrt{\tau/\length} +k(R+B) \triangleq W. \label{eq:vbound}
	\end{equation} By Lemma~\ref{lem:compareDD'}, $v^*$ achieves $O(\epsilon^2)$ squared loss with respect to $\calD$. As the random variable $(\iprod{v^*,z} - y)^2$ for $(z,y)\sim\calD$ is bounded above by \begin{equation}
		(\norm{v^*}\norm{z} + |y|)^2 \lesssim \poly(k,R,B,1/\epsilon,M),
	\end{equation} for $n \ge \poly(k,R,B,1/\epsilon,M)$ we have that the empirical loss of $v^*$ on $(z_1,y_1),\ldots(z_n,y_n)$ is $O(\epsilon^2)$, and therefore that of the predictor $\wt{v}$ is $O(\epsilon^2)$.

	By applying Theorem~\ref{thm:rademacher} with \eqref{eq:zbound} and \eqref{eq:vbound}, we find that the Rademacher complexity $\calR_n(\calF)$ of the family of linear predictors over $\norm{z} \le X$ and with norm bounded by $W$ is $C/\sqrt{n}$ for $C$ which is polynomial in $k$, $R$, $B$, $1/\epsilon$ , $d$, $\log(1/\delta)$, from which the theorem follows by Theorem~\ref{thm:gen}.
\end{proof}

\bibliographystyle{alpha}
\bibliography{biblio}

\newcommand{\etalchar}[1]{$^{#1}$}
\begin{thebibliography}{GKLW18}

\bibitem[ATV21]{awasthi2021efficient}
Pranjal Awasthi, Alex Tang, and Aravindan Vijayaraghavan.
\newblock Efficient algorithms for learning depth-2 neural networks with
  general relu activations.
\newblock {\em arXiv preprint arXiv:2107.10209}, 2021.

\bibitem[BJW19]{bakshi2019learning}
Ainesh Bakshi, Rajesh Jayaram, and David~P Woodruff.
\newblock Learning two layer rectified neural networks in polynomial time.
\newblock In {\em Conference on Learning Theory}, pages 195--268. PMLR, 2019.

\bibitem[CJM20]{CarliniJM20}
Nicholas Carlini, Matthew Jagielski, and Ilya Mironov.
\newblock Cryptanalytic extraction of neural network models.
\newblock In Daniele Micciancio and Thomas Ristenpart, editors, {\em Advances
  in Cryptology - CRYPTO 2020 - 40th Annual International Cryptology
  Conference, CRYPTO 2020, Santa Barbara, CA, USA, August 17-21, 2020,
  Proceedings, Part III}, volume 12172 of {\em Lecture Notes in Computer
  Science}, pages 189--218. Springer, 2020.

\bibitem[CKM20]{CKM}
Sitan Chen, Adam Klivans, and Raghu Meka.
\newblock Learning deep relu networks is fixed parameter tractable.
\newblock {\em arXiv preprint arXiv:2009.13512}, 2020.

\bibitem[CM20]{chen2020learning}
Sitan Chen and Raghu Meka.
\newblock Learning polynomials of few relevant dimensions.
\newblock {\em arXiv preprint arXiv:2004.13748}, 2020.

\bibitem[CS09]{ChoS09}
Youngmin Cho and Lawrence~K. Saul.
\newblock Kernel methods for deep learning.
\newblock In Yoshua Bengio, Dale Schuurmans, John~D. Lafferty, Christopher
  K.~I. Williams, and Aron Culotta, editors, {\em Advances in Neural
  Information Processing Systems 22: 23rd Annual Conference on Neural
  Information Processing Systems 2009. Proceedings of a meeting held 7-10
  December 2009, Vancouver, British Columbia, Canada}, pages 342--350. Curran
  Associates, Inc, 2009.

\bibitem[CW01]{CW01}
A.~Carbery and J.~Wright.
\newblock {Distributional and $L^q$ norm inequalities for polynomials over
  convex bodies in $R^n$}.
\newblock {\em Mathematical Research Letters}, 8(3):233--248, 2001.

\bibitem[DG21]{daniely2021exact}
Amit Daniely and Elad Granot.
\newblock An exact poly-time membership-queries algorithm for extraction a
  three-layer relu network.
\newblock {\em arXiv preprint arXiv:2105.09673}, 2021.

\bibitem[DKKZ20]{diakonikolas2020algorithms}
Ilias Diakonikolas, Daniel~M Kane, Vasilis Kontonis, and Nikos Zarifis.
\newblock Algorithms and sq lower bounds for pac learning one-hidden-layer relu
  networks.
\newblock In {\em Conference on Learning Theory}, pages 1514--1539, 2020.

\bibitem[DV20]{daniely2020hardness}
Amit Daniely and Gal Vardi.
\newblock Hardness of learning neural networks with natural weights.
\newblock {\em arXiv preprint arXiv:2006.03177}, 2020.

\bibitem[DV21]{danielyprg}
Amit Daniely and Gal Vardi.
\newblock From local pseudorandom generators to hardness of learning.
\newblock {\em CoRR}, abs/2101.08303, 2021.

\bibitem[GGJ{\etalchar{+}}20]{goel2020superpolynomial}
Surbhi Goel, Aravind Gollakota, Zhihan Jin, Sushrut Karmalkar, and Adam
  Klivans.
\newblock Superpolynomial lower bounds for learning one-layer neural networks
  using gradient descent.
\newblock {\em arXiv preprint arXiv:2006.12011}, 2020.

\bibitem[GKLW18]{ge2018learning2}
Rong Ge, Rohith Kuditipudi, Zhize Li, and Xiang Wang.
\newblock Learning two-layer neural networks with symmetric inputs.
\newblock In {\em International Conference on Learning Representations}, 2018.

\bibitem[GKM18]{convotron}
Surbhi Goel, Adam~R. Klivans, and Raghu Meka.
\newblock Learning one convolutional layer with overlapping patches.
\newblock In Jennifer~G. Dy and Andreas~Krause 0001, editors, {\em ICML},
  volume~80 of {\em Proceedings of Machine Learning Research}, pages
  1778--1786. PMLR, 2018.

\bibitem[GLM18]{ge2018learning}
Rong Ge, Jason~D Lee, and Tengyu Ma.
\newblock Learning one-hidden-layer neural networks with landscape design.
\newblock In {\em 6th International Conference on Learning Representations,
  ICLR 2018}, 2018.

\bibitem[Han09]{hanneke2009theoretical}
Steve Hanneke.
\newblock {\em Theoretical foundations of active learning}.
\newblock Carnegie Mellon University, 2009.

\bibitem[JCB{\etalchar{+}}20]{JagielskiCBKP20}
Matthew Jagielski, Nicholas Carlini, David Berthelot, Alex Kurakin, and Nicolas
  Papernot.
\newblock High accuracy and high fidelity extraction of neural networks.
\newblock In Srdjan Capkun and Franziska Roesner, editors, {\em 29th USENIX
  Security Symposium, USENIX Security 2020, August 12-14, 2020}, pages
  1345--1362. USENIX Association, 2020.

\bibitem[JSA15]{janzamin2015beating}
Majid Janzamin, Hanie Sedghi, and Anima Anandkumar.
\newblock Beating the perils of non-convexity: Guaranteed training of neural
  networks using tensor methods.
\newblock {\em arXiv}, pages arXiv--1506, 2015.

\bibitem[JWZ20]{jayaram2020span}
Rajesh Jayaram, David~P. Woodruff, and Qiuyi Zhang.
\newblock Span recovery for deep neural networks with applications to input
  obfuscation, 2020.

\bibitem[LMZ20]{LiMaZhang}
Yuanzhi Li, Tengyu Ma, and Hongyang~R. Zhang.
\newblock Learning over-parametrized two-layer relu neural networks beyond ntk.
\newblock {\em CoRR}, abs/2007.04596, 2020.

\bibitem[MSDH19]{MilliSDH19}
Smitha Milli, Ludwig Schmidt, Anca~D. Dragan, and Moritz Hardt.
\newblock Model reconstruction from model explanations.
\newblock In {\em FAT}, pages 1--9. ACM, 2019.

\bibitem[PMG{\etalchar{+}}17]{Papernot}
Nicolas Papernot, Patrick~D. McDaniel, Ian~J. Goodfellow, Somesh Jha, Z.~Berkay
  Celik, and Ananthram Swami.
\newblock Practical black-box attacks against machine learning.
\newblock In Ramesh Karri, Ozgur Sinanoglu, Ahmad-Reza Sadeghi, and Xun Yi,
  editors, {\em Proceedings of the 2017 ACM on Asia Conference on Computer and
  Communications Security, AsiaCCS 2017, Abu Dhabi, United Arab Emirates, April
  2-6, 2017}, pages 506--519. ACM, 2017.

\bibitem[RK20]{RolnickK20}
David Rolnick and Konrad~P. Kording.
\newblock Reverse-engineering deep relu networks.
\newblock In {\em ICML}, volume 119 of {\em Proceedings of Machine Learning
  Research}, pages 8178--8187. PMLR, 2020.

\bibitem[TJ{\etalchar{+}}16]{TramerZJRR16}
Florian Tram{\`e}r, Fan~Zhang 0022, Ari Juels, Michael~K. Reiter, and Thomas
  Ristenpart.
\newblock Stealing machine learning models via prediction apis.
\newblock {\em CoRR}, abs/1609.02943, 2016.

\bibitem[Ver18]{vershynin2018high}
Roman Vershynin.
\newblock {\em High-dimensional probability: An introduction with applications
  in data science}, volume~47.
\newblock Cambridge university press, 2018.

\bibitem[ZSJ{\etalchar{+}}17]{zhong2017recovery}
Kai Zhong, Zhao Song, Prateek Jain, Peter~L Bartlett, and Inderjit~S Dhillon.
\newblock Recovery guarantees for one-hidden-layer neural networks.
\newblock In {\em Proceedings of the 34th International Conference on Machine
  Learning-Volume 70}, pages 4140--4149, 2017.

\bibitem[ZYWG19]{zgu}
Xiao Zhang, Yaodong Yu, Lingxiao Wang, and Quanquan Gu.
\newblock Learning one-hidden-layer relu networks via gradient descent.
\newblock In {\em The 22nd International Conference on Artificial Intelligence
  and Statistics}, pages 1524--1534. PMLR, 2019.

\end{thebibliography}

\end{document}